\RequirePackage[utf8]{inputenc}

\makeatletter
\def\cup@reference@code{%
  \RequirePackage[style=numeric,sorting=none,backend=biber,natbib]{biblatex}%
  \renewcommand*{\bibfont}{\footnotesize}%
}
\makeatother

\documentclass[
  journal=medium,
  manuscript=article,
  manuscriptlabel={Research Preprint},
  logo=false,
  year=2026,
  volume=1
]{cup-journal}

\usepackage[T1]{fontenc}
\usepackage{microtype}
\usepackage{amsmath,amssymb,amsthm}
\usepackage{mathtools}
\usepackage{booktabs}
\usepackage{adjustbox}
\usepackage{makecell,multirow}
\usepackage{orcidlink}
\definecolor{primary}{HTML}{006BA2}
\definecolor{accent}{HTML}{E3120B}
\definecolor{neutral}{HTML}{4A4A4A}

\definecolor{boxfill}{HTML}{E8F4FD}
\definecolor{boxborder}{HTML}{006BA2}

\definecolor{defcolor}{HTML}{2E7D32}
\definecolor{todobg}{HTML}{FFF4E5}
\definecolor{todoframe}{HTML}{B45309}

\usepackage{thmtools}
\usepackage{tcolorbox}
\tcbuselibrary{skins,breakable}
\usepackage{hyperref}
\usepackage{cleveref}

\RawFloats[figure,table]

\ifpdf
  \DeclareGraphicsExtensions{.pdf,.png,.jpg}
\else
  \DeclareGraphicsExtensions{.eps}
\fi

\DeclareMathOperator{\E}{\mathbb{E}}
\DeclareMathOperator{\Var}{Var}
\DeclareMathOperator{\Cov}{Cov}
\DeclareMathOperator{\tr}{tr}
\DeclareMathOperator{\Corr}{Corr}
\DeclareMathOperator{\MMSE}{MMSE}
\DeclareMathOperator{\rank}{rank}
\DeclareMathOperator{\sd}{sd}
\newcommand{\R}{\mathbb{R}}
\newcommand{\norm}[1]{\left\lVert#1\right\rVert}
\newcommand{\abs}[1]{\left|#1\right|}
\newcommand{\HGR}{\rho_{\mathrm{HGR}}}
\newcommand{\alm}{\alpha_m}
\newcommand{\obs}{\mathcal{O}}

\newenvironment{romannum}
  {\begin{enumerate}[label=(\roman*),ref=(\roman*),leftmargin=2.4em,
      labelwidth=1.7em,labelsep=0.5em,align=right]}
  {\end{enumerate}}

\hypersetup{
  colorlinks=true,
  linkcolor=primary,
  citecolor=primary,
  urlcolor=accent,
  bookmarksnumbered=true,
  bookmarksopen=true
}

\crefname{equation}{equation}{equations}
\Crefname{equation}{Equation}{Equations}
\crefname{figure}{Figure}{Figures}
\Crefname{figure}{Figure}{Figures}
\crefname{table}{Table}{Tables}
\Crefname{table}{Table}{Tables}
\crefname{section}{section}{sections}
\Crefname{section}{Section}{Sections}
\crefname{theorem}{theorem}{theorems}
\Crefname{theorem}{Theorem}{Theorems}
\crefname{proposition}{proposition}{propositions}
\Crefname{proposition}{Proposition}{Propositions}
\crefname{corollary}{corollary}{corollaries}
\Crefname{corollary}{Corollary}{Corollaries}
\crefname{conjecture}{conjecture}{conjectures}
\Crefname{conjecture}{Conjecture}{Conjectures}
\crefname{openproblem}{research direction}{research directions}
\Crefname{openproblem}{Research Direction}{Research Directions}

\makeatletter
\@ifpackageloaded{xpatch}{}{\usepackage{xpatch}}
\def\cup@journal@name{Exact Limits of Random Projections}
\def\cup@manuscript{preprint}
\patchcmd{\@maketitle}{\vspace*{\baselineskip}}{\vspace*{0.2\baselineskip}}{}{}
\patchcmd{\@maketitle}{(\cup@year), {\volumefont\cup@vol}, \thepage--\pageref{LastPage}}{}{}{}
\renewcommand*\cup@maketitle@extras@hook{%
  \begingroup
  \renewcommand{\thefootnote}{\fnsymbol{footnote}}%
  \footnotetext{This manuscript has been authored by UT-Battelle, LLC under
  Contract No.\ DE-AC05-00OR22725 with the U.S. Department of Energy. The
  publisher, by accepting the article for publication, acknowledges that the
  United States Government retains a non-exclusive, paid-up, irrevocable,
  world-wide license to publish or reproduce the published form of this
  manuscript, or allow others to do so, for United States Government purposes.
  The Department of Energy will provide public access to these results of
  federally sponsored research in accordance with the DOE Public Access Plan
  (\url{https://www.energy.gov/doe-public-access-plan}).}%
  \endgroup
}

\renewcommand\appendix{\par
  \setcounter{section}{0}%
  \setcounter{subsection}{0}%
  \gdef\thesection{\@Alph\c@section}%
  \gdef\theHsection{\@Alph\c@section}%
  \gdef\theHsubsection{\theHsection.\arabic{subsection}}%
  \gdef\section@cntformat{Appendix \thesection.\quad}%
}
\makeatother

\declaretheoremstyle[
  headfont=\bfseries\color{primary},
  notefont=\normalfont,
  bodyfont=\itshape,
  headpunct={.},
  spaceabove=0pt,
  spacebelow=0pt
]{cupthmplain}
\declaretheoremstyle[
  headfont=\bfseries\color{defcolor},
  notefont=\normalfont,
  bodyfont=\normalfont,
  headpunct={.},
  spaceabove=0pt,
  spacebelow=0pt
]{cupthmdef}
\declaretheoremstyle[
  headfont=\itshape\bfseries\color{neutral},
  notefont=\normalfont,
  bodyfont=\normalfont,
  headpunct={.},
  spaceabove=0pt,
  spacebelow=0pt
]{cupthmrem}

\declaretheorem[style=cupthmplain,name=Theorem,numberwithin=section]{theorem}

\declaretheorem[style=cupthmplain,name=Proposition,sibling=theorem]{proposition}
\declaretheorem[style=cupthmplain,name=Corollary,sibling=theorem]{corollary}
\declaretheorem[style=cupthmplain,name=Conjecture,sibling=theorem]{conjecture}

\declaretheorem[style=cupthmdef,name=Research Direction,sibling=theorem]{openproblem}
\declaretheorem[style=cupthmrem,name=Remark,sibling=theorem]{remark}

\tcbset{
  cupresultbox/.style={
    enhanced,breakable,sharp corners,boxrule=0pt,frame hidden,
    colback=primary!5,borderline west={2pt}{0pt}{primary},
    left=7pt,right=6pt,top=4pt,bottom=4pt,
    before skip=6pt,after skip=6pt
  },
  cupdefinitionbox/.style={
    enhanced,breakable,sharp corners,boxrule=0pt,frame hidden,
    colback=defcolor!6,borderline west={2pt}{0pt}{defcolor},
    left=7pt,right=6pt,top=4pt,bottom=4pt,
    before skip=6pt,after skip=6pt
  },
  cupremarkbox/.style={
    enhanced,breakable,sharp corners,boxrule=0pt,frame hidden,
    colback=neutral!4,borderline west={2pt}{0pt}{neutral},
    left=7pt,right=6pt,top=4pt,bottom=4pt,
    before skip=6pt,after skip=6pt
  }
}
\tcolorboxenvironment{theorem}{cupresultbox}
\tcolorboxenvironment{lemma}{cupresultbox}
\tcolorboxenvironment{proposition}{cupresultbox}
\tcolorboxenvironment{corollary}{cupresultbox}
\tcolorboxenvironment{conjecture}{cupresultbox}
\tcolorboxenvironment{definition}{cupdefinitionbox}
\tcolorboxenvironment{assumption}{cupdefinitionbox}
\tcolorboxenvironment{openproblem}{cupdefinitionbox}
\tcolorboxenvironment{remark}{cupremarkbox}

\title[Exact Limits of Random Projections]{Exact Limits of Random Projections
for Preserving Geometry:\\ Distance Recovery, Nearest-Neighbor Rankings,
and Covariance Shape in Gaussian Models}
\author{Piyush Sao\,\orcidlink{0000-0002-9432-5855}}
\affiliation{Computer Science and Mathematics Division, Oak Ridge National
Laboratory, Oak Ridge, Tennessee 37830, USA}
\email[Piyush Sao]{saopk@ornl.gov}
\hypersetup{
  pdftitle={Exact Limits of Random Projections for Preserving Geometry: Distance Recovery, Nearest-Neighbor Rankings, and Covariance Shape in Gaussian Models},
  pdfauthor={Piyush Sao},
  pdfkeywords={random projection, Johnson--Lindenstrauss lemma, distance recovery, nearest-neighbor ranking, maximal correlation, covariance shape}
}

\begin{document}
\maketitle

\begin{abstract}
The Johnson--Lindenstrauss (JL) lemma guarantees that a random projection of
$n$ points to $m=O(\varepsilon^{-2}\log n)$ dimensions preserves pairwise
squared distances within relative error $\varepsilon$ with high probability,
and this dimension order is asymptotically optimal. In high dimensions,
however, distances concentrate around a baseline while key geometric
information lies in much smaller fluctuations. We show that the JL bound can
therefore be uninformative about retained geometry: an independent Gaussian
replacement map can satisfy it even though the replacement cloud is
independent of the original data.

We then ask how well any decoder can recover a feature $f(D)$ of a squared
distance $D$ from a linear sketch. Under squared-error loss, the optimal decoder
is conditional expectation, so recovery defines a linear operator whose
singular values quantify feature recovery. For isotropic Gaussian data
($\Sigma=\sigma^2 I_d$), we diagonalize this operator in closed form. For fixed
$k$ with $m,d-m\to\infty$, its $k$th singular value satisfies
$\ell_k\approx(m/d)^{k/2}$.

This yields three sharp consequences. A rank-$m$ sketch retains at most an
$m/d$ fraction of the variance of any feature of one squared distance. If
$m\to\infty$ and $m/d\to0$, the expected Kendall correlation is
$\frac{2}{\pi}\sqrt{m/d}(1+o(1))$; for fixed $q$, nearest-neighbor agreement
tends to $1/q$. Yet one projection can satisfy the JL bound while mean Kendall
correlation vanishes when $\log n\ll m\ll d$. After removing scale,
Haar-averaged retained covariance-shape information is $(m/d)^2$. Thus JL
distance preservation does not quantify the geometry available for comparison
or inference.

\end{abstract}

\section{Introduction}
\label{sec:intro}

High-dimensional data pipelines routinely replace raw data with compressed
representations. Approximate-nearest-neighbor systems generate candidates in
a compressed domain and rerank them with richer codes or full-precision
distances \cite{johnson2021faiss,jegou2011rerank}. Single-cell methods
summarize high-dimensional measurements in low-dimensional visualizations,
even as benchmarks warn that the recovered orderings can be unstable or
distorted \cite{moon2019phate,saelens2019trajectory,chari2023specious}.
The decisions this forces on the practitioner---when to rerank, how far to
compress, how many independent sketches to draw---are currently made
empirically, because available guarantees may not
  capture the task-relevant notion of quality, and their worst-case bounds can be
  too conservative to guide parameter selection \cite{martinsson2020randomized}.
The goal of this paper is to quantify
that retained geometry for the random projection, the compressed
representation with the strongest formal guarantees
\cite{halko2011finding,indyk1998approximate}.

The guarantee behind random projection is the Johnson--Lindenstrauss (JL)
lemma \cite{johnson1984extensions}: for finite sets, a suitable map into
dimension $m=O(\varepsilon^{-2}\log n)$ preserves every pairwise squared
distance among $n$ points to relative error $\varepsilon$. For an indexed
sample $x_1,\ldots,x_n$, we say that a map $T$ satisfies the
\emph{$\varepsilon$-JL bound} when
\[
(1-\varepsilon)\norm{x_i-x_j}^2
\le \norm{T(x_i)-T(x_j)}^2
\le (1+\varepsilon)\norm{x_i-x_j}^2
\qquad\text{for every }i<j.
\]
If the sample or map is random, we call the occurrence of this bound the
\emph{JL event}.

For concentrated high-dimensional data, \emph{concentration of measure}
causes the pairwise squared distances controlled by the JL bound to be
dominated by a common population mean, which we call the \emph{baseline}.
Consequently, these bounds become less informative about specific pairwise
distances, which instead appear as smaller centered \emph{fluctuations}
around this baseline. Downstream tasks depend on this fine scale: which
distance is smaller, and how far a point distribution departs from
sphericity. Two questions follow, and the JL bound answers neither. Does
the JL event itself certify that the geometry carrying this fine scale
survived the sketch? And how much of the fluctuations can any decoder
recover from a rank-$m$ sketch?

The structure of prior work indicates why both questions are open. Lower
bounds show that the dimension order in the JL lemma is essentially
unimprovable: Alon \cite{alon2003problems} showed that near-equidistant
point sets already force nearly the same dimension order for every map,
linear or not, and Larsen and Nelson \cite{larsen2017optimality} sharpened
this lower bound to match the JL dimension order for worst-case
constructions. Because the lemma attains the optimal dimension order for
the guarantee it states, the gap cannot be closed by sharpening the lemma
itself; informativeness about concentrated data requires a different
criterion. The literature on concentration addresses the original space
rather than the sketch: Beyer et al.\ \cite{beyer1999nearest} gave a
sufficient condition for distance concentration, and the Durrant--Kab\'an
converse \cite{durrant2009nearest} established the corresponding necessary
condition and identified when nearest neighbors remain meaningful. A third
line characterizes projection outputs: Dasgupta
\cite{dasgupta1999learning,dasgupta2000experiments} established what we
call \emph{Dasgupta's sphericalization phenomenon}, the rounding of
eccentric Gaussian components toward sphericity; Dasgupta, Hsu, and Verma
\cite{dasgupta2006concentration} and Diaconis and Freedman
\cite{diaconis1984asymptotics} obtained related structural results for
broader distributions, including approximation by scale mixtures of
spherical Gaussians; and Li and Malik \cite{li2016fast} gave order
preservation guarantees for fixed vector pairs at prescribed length ratios.
None of these results tests whether JL satisfaction certifies retained
geometry, and none quantifies how much distance variation a decoder can
recover from a sketch.

Our approach introduces the missing object: an explicit decoder. We model
recovery as estimation of a feature of a squared distance from the sketch
under squared-error loss. The optimal estimate is the conditional
expectation, so recovery acts as a linear operator, and its singular values
measure how much of each feature survives compression. This formulation
yields information limits that the pairwise error criterion of the JL lemma
cannot express. We make three main contributions. First, we prove that the
error allowed by the JL bound can exceed the fluctuation scale of a
concentrated Gaussian cloud (\cref{sec:scale}). We then show that an
independent Gaussian replacement map can satisfy the same bound at the
dimension prescribed by the JL lemma even though its replacement cloud is
independent of the data (\cref{sec:certification}); satisfying the JL bound
therefore need not certify retained geometry. Second, we identify three
recovery targets that the JL bound does not separately control---distance
values, neighbor rankings, and covariance shape---and derive their exact
Gaussian recovery limits (\crefrange{sec:question}{sec:rate3}). For
isotropic distance features, the recoverable variance ceiling is $m/d$;
neighbor rankings vary on the associated correlation scale $\sqrt{m/d}$;
and diffuse covariance shape contracts at order $(m/d)^2$. Third, for a
general known map, we separate what depends on rank, singular spectrum, and
numerical precision, then quantify what ensembles of independent sketches
can recover (\cref{sec:generalmaps,sec:ensembles}).

These laws give the practitioner's empirical decisions a quantitative
scale. Within the Gaussian benchmark, the fraction of distance variance
retained, the correlation scale of neighbor rankings, and the contraction
of covariance shape are each determined by the ratio $m/d$, so the cost of
compressing farther and the benefit of adding independent sketches can be
read off before any experiment. The replacement-map result adds a caution
that no empirical protocol supplies: a map can satisfy the JL bound while
its output is independent of the data. Beyond the specific Gaussian model,
we expect these techniques to extend to other distributions and to more
complex features; \cref{sec:conjectures} develops the resulting research
directions. Except where a complete argument appears directly after a
statement, proofs are collected in the supplement (\cref{app:proofs}); the
discussion following a statement in the text records the mechanism of the
proof.

\section{The Scale Problem: Baseline versus Fluctuations}
\label{sec:scale}

We begin by quantifying when the relative error allowed by the JL bound is too
coarse to resolve distance fluctuations.
Let $X,X'\overset{\text{i.i.d.}}{\sim}\mathcal N(\mu,\Sigma)$ in $\R^d$,
where the eigenvalues of $\Sigma$ are
$\lambda_1\ge\cdots\ge\lambda_d\ge0$. Let
$L:\R^d\to\R^m$ be any linear map of rank at most $m$, random
or deterministic, oblivious or $\Sigma$-aware; no statement below depends on how the map is chosen. Write
\[
D=\norm{X-X'}^2,
\qquad
\obs=(LX,LX')
\]
for the squared distance and the observed sketch.
Whenever $L$ is random, we condition on its realization and reveal it to the
decoder; equivalently, the information statements below are conditional on
the sampled map. For a query $X_0$ and
two candidates $X_1,X_2$, define the \emph{neighbor contrast}
\[
S=\norm{X_0-X_1}^2-\norm{X_0-X_2}^2.
\]
Its sign identifies the nearer candidate, so preserving a two-candidate
ranking is equivalent to preserving the sign of $S$.

The baseline and the fluctuations appear in the first two moments:
\begin{equation}
\E D=2\tr\Sigma,
\qquad
\Var(D)=8\tr(\Sigma^2).
\label{eq:moments}
\end{equation}
The mean $\E D$ is the common population \emph{baseline} for i.i.d. pairs, while
the centered deviations $D-\E D$ are the \emph{fluctuations} that carry
comparative geometry. Diagonalizing $\Sigma$ expresses $D$ as a weighted sum
of independent chi-square variables. Their variances add, so the fluctuation
scale depends on the squared eigenvalues through $\tr(\Sigma^2)$. The relative
scale of the fluctuations is
\begin{equation}
r_2(\Sigma)=\frac{(\tr\Sigma)^2}{\tr(\Sigma^2)},
\qquad
\frac{\sd(D)}{\E D}=\sqrt{\frac{2}{r_2(\Sigma)}}.
\label{eq:r2}
\end{equation}
The quantity $r_2(\Sigma)$ is a squared-trace effective rank: larger values
mean tighter relative concentration. For an isotropic covariance matrix,
$r_2(\Sigma)=d$, so relative fluctuations are of order $d^{-1/2}$
\cite{beyer1999nearest}.

The following proposition states exactly when a relative-error
bound fixes the sign of a comparison. Here $\widetilde D_1$ and
$\widetilde D_2$ are any distorted values satisfying the displayed
inequalities; later sections reserve the notation $\widetilde D$ for the
squared distance computed from a sketch.
\begin{proposition}[When a relative-error bound preserves order]
\label{prop:jl-margin}
Let $0<D_1<D_2$, and suppose
\[
(1-\varepsilon)D_j\le \widetilde D_j\le(1+\varepsilon)D_j,
\qquad j=1,2.
\]
These inequalities guarantee $\widetilde D_1<\widetilde D_2$ for every
admissible pair if and only if
\[
\frac{D_2-D_1}{D_1+D_2}>\varepsilon.
\]
\end{proposition}

\begin{proof}
The worst admissible distortions preserve the order exactly when
$(1+\varepsilon)D_1<(1-\varepsilon)D_2$, which is equivalent to the displayed
condition; if this inequality fails, the extreme admissible errors reverse or
tie the order.
\end{proof}

The neighbor contrast is one such comparison, with
$D_j=\norm{X_0-X_j}^2$: certifying which candidate is nearer is certifying the
sign of $S=\pm(D_2-D_1)$. Comparisons with smaller margins may survive, but
the bound does not guarantee them. \Cref{sec:rate2} quantifies their average
behavior for the neighbor contrast under random projection.

The allowed absolute error in a typical squared distance is of order
$\varepsilon\E D$. If
$\varepsilon\gg r_2(\Sigma)^{-1/2}$, this error exceeds the fluctuation scale
$\sd(D)$. In this regime, the JL bound can still hold, but it does not by
itself certify neighbor comparisons. If instead
$\varepsilon\asymp r_2(\Sigma)^{-1/2}$, the JL lemma gives
$m=O(r_2(\Sigma)\log n)$. For isotropic data, this scales as $d\log n$ and
therefore no longer guarantees a target dimension below $d$.

The allowed error can therefore be too large to resolve
comparative geometry. A coarser tolerance still leaves a distinct question:
whether passing the test certifies anything about the retained geometry.

\section{The JL Bound Can Hold without Information}
\label{sec:certification}

\paragraph{Does satisfying the JL bound imply retention of any sample-specific structure?}
\Cref{sec:scale} showed that the error allowed by the JL bound can exceed the
fluctuation scale. That observation concerns the tolerance, not the information
retained by a particular map: a loose bound may coexist with informative maps.
We therefore ask: does satisfying the JL bound guarantee any
sample-specific structure?

We answer negatively by constructing a map. An \emph{independent Gaussian
replacement map} assigns each indexed input point $X_i$ a fresh Gaussian
output $Y_i$, sampled independently of the data, so that
$I(X^n;Y^n)=0$ and no decoder can extract any feature of the input from the
output. Its image is the \emph{replacement cloud}; the shared indices pair
each input distance $\norm{X_i-X_j}$ with the corresponding output distance
$\norm{Y_i-Y_j}$. The question is whether this zero-information map can
nonetheless satisfy the $\varepsilon$-JL bound on all $\binom n2$ pairs at
the standard JL dimension.

\subsection{The independent Gaussian replacement map}
\begin{proposition}[A zero-information map can satisfy the JL bound]
\label{prop:zeroinfo}
Let $n\ge2$, $m,d\in\mathbb N$, $\sigma>0$, and
$0<\varepsilon,\delta<1$. Let
$X_1,\dots,X_n\overset{\mathrm{iid}}{\sim}\mathcal N(\mu,\sigma^2I_d)$ and
$Y_1,\dots,Y_n\overset{\mathrm{iid}}{\sim}
\mathcal N\!\bigl(0,\sigma^2\tfrac dm I_m\bigr)$ be independent, and define
the index map $T_{\mathrm{rep}}(X_i)=Y_i$. If
$\min\{m,d\}\ge72\,\varepsilon^{-2}
\log\!\bigl(4\tbinom n2/\delta\bigr)$, then, with probability at least
$1-\delta$,
\[
(1-\varepsilon)\norm{X_i-X_j}^2
\le \norm{T_{\mathrm{rep}}(X_i)-T_{\mathrm{rep}}(X_j)}^2
\le(1+\varepsilon)\norm{X_i-X_j}^2
\qquad\text{for all }i<j.
\]
In particular, in the compression regime $m\le d$, $T_{\mathrm{rep}}$ can
satisfy the $\varepsilon$-JL bound on the sample at the standard JL dimension
order $m=O(\varepsilon^{-2}\log(n/\delta))$, while its output carries no
information about the data: $I(X^n;Y^n)=0$.
\end{proposition}

\begin{proof}
For every $i<j$,
\[
\frac{\norm{X_i-X_j}^2}{2\sigma^2d}
\overset{d}{=}\frac{\chi^2_d}{d},
\qquad
\frac{\norm{Y_i-Y_j}^2}{2\sigma^2d}
\overset{d}{=}\frac{\chi^2_m}{m}.
\]
Set $a=\varepsilon/(2+\varepsilon)$. The Chernoff bound
gives $\Pr(\abs{\chi_k^2/k-1}>a)\le2e^{-ka^2/8}$. A union bound over the
$2\binom n2$ normalized squared distances places all of them in
$[1-a,1+a]$ with probability at least
\[
1-4\binom n2
\exp\!\left(-\frac{\min\{m,d\}a^2}{8}\right)
\ge1-\delta,
\]
where the last inequality follows from the dimension hypothesis and
$8(2+\varepsilon)^2<72$. On this event, every output-to-input ratio lies in
\[
\left[\frac{1-a}{1+a},\frac{1+a}{1-a}\right]
=\left[\frac1{1+\varepsilon},1+\varepsilon\right]
\subseteq[1-\varepsilon,1+\varepsilon],
\]
which proves the claim.
\end{proof}

The mechanism is that the JL event factors: input-cloud concentration and
output-cloud concentration around a shared baseline suffice, and no
cross-conditional structure is needed. The following remark records the
constants and the two-stage form.

\begin{remark}[Sharp constants, two-stage form, and scope]
\label[remark]{rem:zeroinfo-details}
Let $\mathcal E_\varepsilon^{\mathrm{rep}}$ denote the JL event in
\cref{prop:zeroinfo}, and, for $0<a<1$, define
\[
\mathcal G_X(a)=
\left\{\left|\frac{\norm{X_i-X_j}^2}{2\sigma^2d}-1\right|\le a
\ \text{for all }i<j\right\}.
\]
\emph{(a)} In the compression regime $m\le d$, the hypothesis is a condition
on $m$ alone, with the dimension order prescribed by the JL lemma:
$m=O(\varepsilon^{-2}\log(n/\delta))$.

\emph{(b)} The proof gives the sharp form: with
$a_\varepsilon=\varepsilon/(2+\varepsilon)$,
\[
\Pr(\mathcal E_\varepsilon^{\mathrm{rep}})
\ge1-2N_pe^{-da_\varepsilon^2/8}-2N_pe^{-ma_\varepsilon^2/8},
\]
and
$\min\{m,d\}\ge8(2+\varepsilon)^2\varepsilon^{-2}
\log(4N_p/\delta)$ suffices.

\emph{(c)} The same proof yields the two-stage bounds
\begin{equation}
\label{eq:twostage}
\Pr_X\bigl(\mathcal G_X(a_\varepsilon)\bigr)
\ge1-2N_pe^{-da_\varepsilon^2/8},
\qquad
\inf_{x^n\in\mathcal G_X(a_\varepsilon)}
\Pr_Y\bigl(\mathcal E_\varepsilon^{\mathrm{rep}}\mid X^n=x^n\bigr)
\ge1-2N_pe^{-ma_\varepsilon^2/8},
\end{equation}
so most fixed input clouds individually admit an independent Gaussian
replacement map. The phenomenon is therefore a property of concentration,
not of averaging over random inputs.

\emph{(d)} Independence concerns the unconditional joint law; conditioning on
the success event induces dependence.
\end{remark}

\Cref{fig:zero-info-jl} illustrates the separation. For comparison, we use a
rescaled rank-$m$ orthogonal projection whose row space is uniformly random
(Haar-distributed, in the terminology of \cref{sec:channel}).
\begin{figure}[H]
\centering
\includegraphics[width=\linewidth]{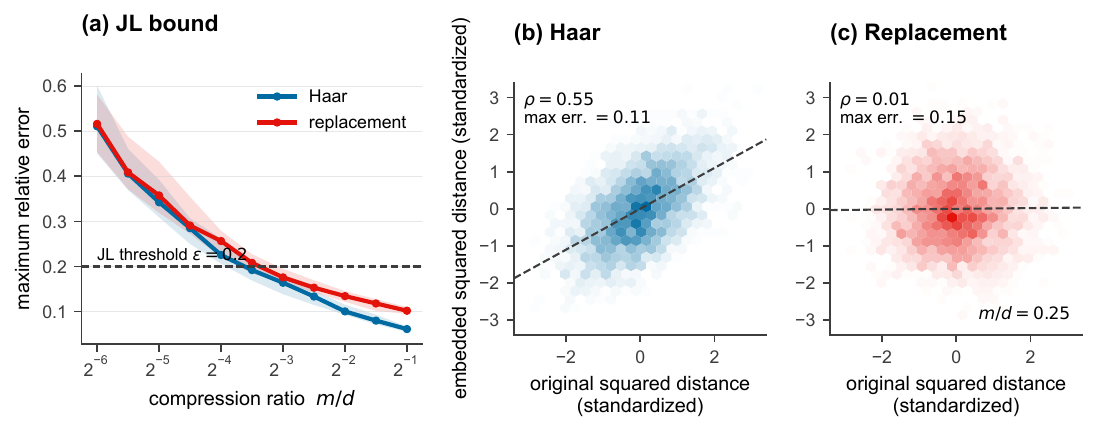}
\caption{Satisfying the JL bound does not certify retained information. Panel
(a) reports the median maximum relative error over all
$\binom{100}{2}$ pairs for a rescaled uniformly random rank-$m$ orthogonal
projection and an independent Gaussian replacement map at $d=8192$; bands
span the $10$th--$90$th percentiles over $30$ trials, and the dashed line marks
$\varepsilon=0.2$. Panels (b)--(c) show one trial at $m/d=1/4$: projected
squared distances remain correlated with the original squared-distance
fluctuations, whereas replacement squared distances do not. The annotated
maximum errors show that both
displayed maps satisfy the JL bound.}
\label{fig:zero-info-jl}
\end{figure}

Three natural objections delimit the scope of this separation: what the
output actually tells a practitioner, whether the construction depends on
population scaling, and whether it requires an oversized dimension. The
same independence gives an exact answer to the first.
\begin{corollary}[Ranking consequences of the replacement map]
\label{cor:zero-info-consequences}
Under the assumptions of \cref{prop:zeroinfo}, if $n\ge3$, define the
mean Kendall statistic for the replacement map by
\[
\overline\tau_n^{\mathrm{rep}}
=\frac1n\sum_{i=1}^n\binom{n-1}{2}^{-1}
\sum_{\substack{j<k\\j,k\ne i}}
\operatorname{sgn}\!\left[
(D_{ij}-D_{ik})(D_{ij}^{\mathrm{rep}}-D_{ik}^{\mathrm{rep}})
\right].
\]
Then
\[
\E\overline\tau_n^{\mathrm{rep}}=0,
\qquad
\Pr\!\left(\abs{\overline\tau_n^{\mathrm{rep}}}>t\right)
\le2\exp\!\left(-\frac{\lfloor n/3\rfloor t^2}{2}\right).
\]
In the analogous setting with one query and $q\ge2$ candidates, let
$\mathcal N_k$ and $\mathcal N_k^{\mathrm{rep}}$ be the top-$k$ sets for the
original and replacement clouds, respectively, where $1\le k<q$. Then
\[
\E\frac{\abs{\mathcal N_k\cap\mathcal N_k^{\mathrm{rep}}}}{k}
=\frac{k}{q},
\qquad
\Pr\!\left(\mathcal N_1=\mathcal N_1^{\mathrm{rep}}\right)=\frac1q.
\]
\end{corollary}

\begin{proof}
Each Kendall summand is a product of two independent signs. Exchangeability of the
candidate labels makes one sign symmetric, so the kernel has mean zero. Since
$(X_i,Y_i)$ are i.i.d. pairs across $i$, $\overline\tau_n^{\mathrm{rep}}$ is a
bounded order-three U-statistic. Hoeffding
\cite{hoeffding1963probability} proved the required blocking inequality;
applying it gives the stated tail bound.
The neighbor set from the replacement cloud is a function of that cloud alone;
by exchangeability it is a uniform $k$-subset independent of
$\mathcal N_k$. Ties have probability zero.
\end{proof}

\subsection{Calibration}

\paragraph{Empirical calibration.}
Empirical calibration can replace population scaling, which uses knowledge
of $\sigma^2$ that a practitioner may not have. With
$\overline D_X=N_p^{-1}\sum_{i<j}D_{ij}$ and
$\overline D_{\mathrm{rep}}=N_p^{-1}\sum_{i<j}D_{ij}^{\mathrm{rep}}$,
the rescaled cloud
$Y_i^\star=(\overline D_X/\overline D_{\mathrm{rep}})^{1/2}Y_i$ matches the
realized mean squared distance exactly. The same proof holds with
$b_\varepsilon=(\sqrt{1+\varepsilon}-1)/(\sqrt{1+\varepsilon}+1)
\sim\varepsilon/4$ in place of $a_\varepsilon$, because on the band event
both empirical means also lie in the band and the worst ratio is
$((1+b_\varepsilon)/(1-b_\varepsilon))^2=1+\varepsilon$. Global rescaling
preserves all orderings, so the conclusions of
\cref{cor:zero-info-consequences} for rankings and neighborhoods remain
exact. If $\mathcal E_\varepsilon^\star$ denotes the calibrated JL event,
then
\[
\Pr(\mathcal E_\varepsilon^\star)
\ge1-2N_pe^{-db_\varepsilon^2/8}-2N_pe^{-mb_\varepsilon^2/8}.
\]
The calibrated output is not independent of the input. Its data-dependent
part is characterized precisely by
\[
X^n\longrightarrow\overline D_X\longrightarrow Y^{\star n},
\qquad
\overline D_{Y^\star}=\overline D_X\quad\text{almost surely}.
\]
Thus the conditional law of the output depends on $X^n$ only through
$\overline D_X$, and that statistic is recoverable from the output: empirical
calibration leaks exactly one scalar and no ranking information.

\subsection{Sharpness}

The third objection is that the replacement map might succeed only in a
dimension so large that the JL bound becomes trivial to satisfy. The
following proposition rules this out.

\begin{proposition}[Sharp dimension order for independent Gaussian replacement maps]
\label{prop:zero-info-sharpness}
In the setting of \cref{prop:zeroinfo}, let
$r=\lfloor n/2\rfloor$ and $0<\varepsilon\le1$. There are universal
constants $c,C>0$ such that
\[
\Pr(\mathcal E_\varepsilon^{\mathrm{rep}})
\le \exp\!\left(-cr e^{-Cm\varepsilon^2}\right).
\]
Consequently, if
$\Pr(\mathcal E_\varepsilon^{\mathrm{rep}})\ge1-\delta$, then
\[
m\ge\frac1{C\varepsilon^2}
\log\!\left(\frac{cr}{-\log(1-\delta)}\right)
\]
whenever the logarithm is positive. In particular, for $\delta\le1/2$ and
$n/\delta$ sufficiently large,
$m=\Omega(\varepsilon^{-2}\log(n/\delta))$.
\end{proposition}

The proof in \cref{app:proofs} uses the i.i.d. $F_{m,d}$ ratios associated with
disjoint pairs and the Zhang--Zhou lower bound
\cite[Cor.~3]{zhang2020nonasymptotic} for the upper chi-square tail
probability. Together with \cref{prop:zeroinfo}, this shows that the dimension
order prescribed by the JL lemma is both sufficient and necessary for the
independent Gaussian replacement map.

The conclusion is that the certificate and the information are separate
objects. Satisfying the JL bound does not by itself certify retained
geometry; what replaces the certificate is the subject of the next section.

\section{Recovery as an Operator Problem}
\label{sec:question}

\subsection{Sketch, target, and decoder}
\label{sec:projection}

The JL bound checks if a map preserves original distances within a small relative error. In contrast, recovery analysis fixes an observation and determines what any decoder can infer from it. \Cref{sec:scale} and
\cref{sec:certification} show why this distinction matters: the JL bound can hold even when the mapped points contain no information about the original geometry.

\begin{quote}
\textbf{Central question.} Which features of the original geometry can a decoder recover from the sketch, and how much of each?
\end{quote}

To analyze a single distance, we formulate the question as follows. A known linear map $L:\R^d\to\R^m$ acts on a pair $X,X'$. The target distance is $D=\norm{X-X'}^2$, and the observation is the sketch $\obs=(LX,LX')$. A
\emph{feature} is a square-integrable scalar $f(D)$ with positive variance $\Var(f(D))>0$. A \emph{decoder} is any measurable function $g(\obs)$.

Under squared-error loss, the optimal decoder for a feature is its conditional expectation. By the law of total variance, this optimal decoder minimizes the mean squared error:
\[
g_f^\star(\obs)=\E[f(D)\mid\obs],
\qquad
\inf_g\E\bigl[(f(D)-g(\obs))^2\bigr]
=\Var(f(D))-\Var\bigl(\E[f(D)\mid\obs]\bigr).
\]
The conditional expectation orthogonally projects $f(D)$ onto the closed subspace of square-integrable functions of the sketch. This projection splits the feature's variance into a recovered part and an orthogonal residual. We
measure this recovery using the \emph{recoverable variance fraction}:
\[
\mathsf{Rec}(f;\obs)
=\frac{\Var\bigl(\E[f(D)\mid\obs]\bigr)}{\Var(f(D))}\in[0,1].
\]
This fraction depends only on the feature and the observation, not on the choice of recovery algorithm.

The criterion separates what the JL bound conflates. For the replacement
cloud of \cref{prop:zeroinfo}, independence gives
$\E[f(D)\mid Y^n]=\E f(D)$, hence $\mathsf{Rec}(f;Y^n)=0$ for every
feature: the JL event can coexist with nothing recoverable. A linear sketch
of the data instead yields a nonzero projection. The supremum of
$\mathsf{Rec}(f;\obs)$ over features is therefore the object of interest,
and it is an operator norm.

\subsection{The distance-recovery operator}
\label{sec:operator}

Gathering all centered, square-integrable features of $D$ into the space $L^2_0(D)$ frames recovery as an operator problem. We define the \emph{recovery operator}
\[
\mathcal A:L^2_0(D)\longrightarrow L^2_0(\obs),
\qquad
\mathcal Af=\E[f(D)\mid\obs],
\]
where $L^2_0(\obs)$ denotes the centered, square-integrable functions of the observation. Because $\mathcal A$ is a restricted orthogonal projection, it never increases the $L^2$ norm, yielding
\[
\mathsf{Rec}(f;\obs)=\frac{\norm{\mathcal Af}_2^2}{\norm{f}_2^2}.
\]
The squared operator norm $\norm{\mathcal A}_{\mathrm{op}}^2$ represents the maximum recoverable fraction over all features. This norm also equals the squared Hirschfeld--Gebelein--R\'enyi (HGR) maximal correlation
\cite{renyi1959measures,dembo2001remarks}, which measures the maximum correlation obtainable by transforming the target and the observation:
\begin{equation}
\norm{\mathcal A}_{\mathrm{op}}^2
=\sup_{\substack{f\in L^2_0(D)\\ f\ne0}}
\frac{\norm{\mathcal Af}_2^2}{\norm{f}_2^2}
=\HGR^2(D;\obs),
\qquad
\HGR(Y;Z)=\sup_{f,g}\abs{\Corr(f(Y),g(Z))}.
\label{eq:bg-hgr}
\end{equation}
Here, the last supremum is over all measurable, nonconstant, square-integrable scalar functions. Optimizing first over $g$ yields the normalized conditional variance, which confirms this identity \cite{renyi1959measures}.
Consequently, no estimator can recover a variance fraction of any feature of $D$ that exceeds $\norm{\mathcal A}_{\mathrm{op}}^2$.

The operator describes recovery at three levels:
\[
\begin{array}{ll}
\text{one feature:} & \norm{\mathcal Af}_2^2/\norm{f}_2^2,\\[2pt]
\text{the best feature:} & \norm{\mathcal A}_{\mathrm{op}}^2,\\[2pt]
\text{all features at once:} & \text{the singular system of }\mathcal A.
\end{array}
\]
Constant features are recovered perfectly and form a trivial mode. Centering removes this mode, ensuring that the operator on $L^2_0(D)$ captures only the geometric information. While the operator norm is always defined, a
discrete singular system exists only for specific channels. The Gaussian channel we construct next admits such a system in closed form.

\subsection{The Gaussian distance channel}
\label{sec:channel}

We identify the concrete operator by reducing the sketch to a scalar
channel. The first step, \emph{row orthonormalization}, is a reversible
output transformation that preserves all information in a noiseless linear
observation by a known map. Let a rank-$r$ map $L$ have thin singular value
decomposition $L=U_rS_rV_r^\top$ and put $R=V_r^\top$. Discarding redundant
output coordinates yields
\[
Lx=U_rS_rRx,
\qquad
Rx=S_r^{-1}U_r^\top Lx.
\]
Because $S_r$ is invertible, $Lx$ and $Rx$ uniquely determine each other on
the nonredundant output. If $L$ has full row rank, we can instead take
$R=(LL^\top)^{-1/2}L$. Here $r$ denotes the rank of $L$; for
full-row-rank sketches with $m$ rows, $r=m$.

Let $\Sigma=\sigma^2I_d$, and let $L$ have rank $0<r<d$. Under this isotropic
covariance, row orthonormalization replaces the sketch with the equivalent
pair $(\Pi X,\Pi X')$, where $\Pi$ is a rank-$r$ orthogonal projector, and
by rotational invariance the law of this pair depends on $\Pi$ only through
its rank. For a random oblivious map, the row space is therefore all that
matters. A random $r$-dimensional subspace of $\R^d$ is \emph{Haar
distributed} if it is uniform over the Grassmannian, the set of all such
subspaces; represent it by an $r\times d$ matrix $R$ with orthonormal rows
and projector $\Pi=R^\top R$. Orthonormalizing the rows of an i.i.d.
standard Gaussian matrix yields this same Haar row-space distribution, a
standard construction in randomized numerical linear algebra
\cite{halko2011finding}: right multiplication by a fixed orthogonal matrix
preserves the Gaussian law, which makes the row space uniform on the
Grassmannian.

Standardize the difference as
\[
Z=\frac{X-X'}{\sqrt2\sigma}\sim\mathcal N(0,I_d),
\qquad
\mathcal T=\norm{Z}^2,
\qquad
U=\norm{\Pi Z}^2,
\qquad
V=\mathcal T-U.
\]
Then $U\sim\chi_r^2$, $V\sim\chi_{d-r}^2$, $U\perp V$, and
$D=2\sigma^2\mathcal T$.

Three independence properties reduce the full sketch to the scalar $U$.
First, the Gaussian midpoint $(X+X')/2$ is independent of $X-X'$, so its
projection carries no information about $D$. Second, rotational invariance
makes the direction of $\Pi Z$ uniform on the unit sphere of
$\operatorname{range}(\Pi)$ and independent of its length $U^{1/2}$. Third,
the orthogonal Gaussian components $\Pi Z$ and $(I-\Pi)Z$ are independent,
so $V$ is independent of the entire retained vector. Therefore, the
conditional law of $\mathcal T$ given the full sketch depends only on $U$;
in probabilistic notation, $D\perp\obs\mid U$.
\Cref{fig:channel} summarizes this reduction.

Standard gamma-sum identities now specify the channel. If
$U\sim\operatorname{Gamma}(a,\theta)$ and
$V\sim\operatorname{Gamma}(b,\theta)$ are independent, then
\[
\mathcal T=U+V\sim\operatorname{Gamma}(a+b,\theta),
\qquad
\frac{U}{\mathcal T}\sim\operatorname{Beta}(a,b),
\qquad
\frac{U}{\mathcal T}\perp\mathcal T.
\]
Here $a=r/2$, $b=(d-r)/2$, and the common scale is $\theta=2$ for the
chi-squared laws above. Observing $U$ from the latent total
$\mathcal T=U+V$ forms the \emph{beta--gamma channel}.

\begin{figure}[t]
\centering
\includegraphics[width=.82\linewidth]{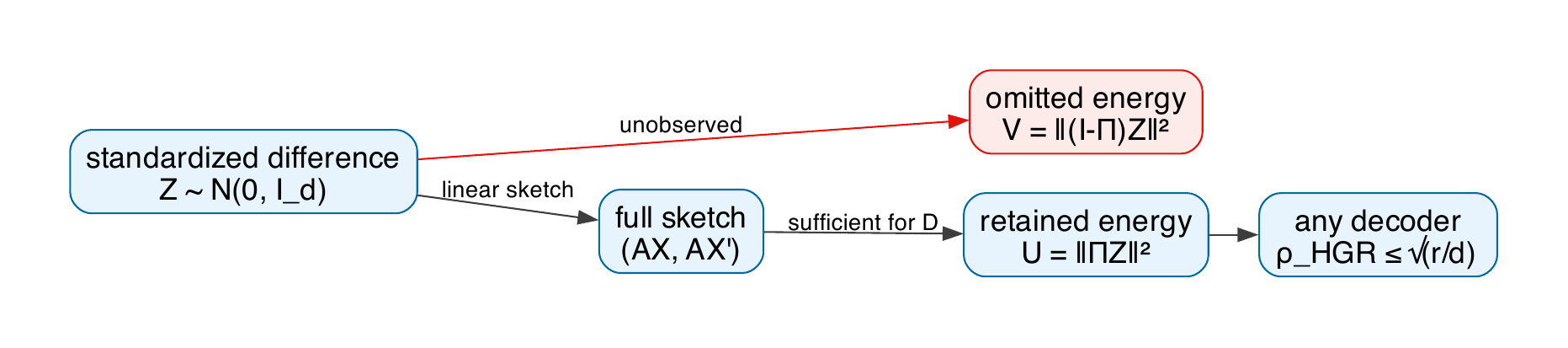}
\caption{The distance channel. A rank-$r$ sketch splits the squared norm of a
standardized Gaussian difference into observed and omitted parts. Given $U$,
the rest of the sketch provides no further information about the
distance $D=2\sigma^2(U+V)$.}
\label{fig:channel}
\end{figure}

The distance-recovery problem is now reduced to the conditional-expectation
operator of the beta--gamma channel: since $D\perp\obs\mid U$ and
$D=2\sigma^2\mathcal T$, the operator $\mathcal A$ acts through
$h(\mathcal T)\mapsto\E[h(\mathcal T)\mid U]$. \Cref{sec:hgr} computes its
operator norm and its complete singular spectrum.

\subsection{How the framework recurs}
\label{sec:targets}

The ranking and covariance-shape problems replace the pair
$(L^2_0(D),L^2_0(\obs))$ by other target and observable spaces. One
formulation---conditional expectation as orthogonal projection---covers all
three problems; the operator itself differs with each target. Throughout,
$d$ is the ambient dimension, $m$ the output dimension, and $r$ the map
rank; $n$ is the sample size and $q$ the number of candidates compared to a
query; $D$ denotes a squared distance, $S$ a shared-query distance
contrast, and $\obs$ the observed sketch. Three derived quantities recur:
$\HGR$ is the maximal correlation between a target and an observation,
$r_2(\Sigma)=(\tr\Sigma)^2/\tr(\Sigma^2)$ is the effective rank of the
spectrum, and $\ell_k$ indexes the singular values of the distance-recovery
operator. Sections~\ref{sec:rate1} and beyond introduce the recovery
fractions $\alm(\Sigma)$ and $\alpha_\Pi(\Sigma)$ and the block retention
ratios $\theta_g$ where they are first used.

For distance features, \cref{sec:hgr} solves the operator just constructed:
the centered distance attains the operator norm $\sqrt{r/d}$, and the full
singular system is a generalized Laguerre family.
\Cref{sec:rate1} then treats general covariance, where recovering the
distance value and recovering the best nonlinear feature can differ. The
recovery scale is a variance fraction of order $m/d$, a quantity the
JL bound does not control (\cref{sec:scale}).

For neighbor rankings, the target is the shared-query contrast
$S=D_{01}-D_{02}$, or its sign, and we observe
$\obs=(LX_0,LX_1,LX_2)$. The target is a function of two dependent distances,
so the scalar operator is replaced by a projection of the joint contrast; we
calculate the exact sign law from a conditional Gaussian calculation
(\cref{sec:rate2}). Sign statistics inherit the correlation scale
$\sqrt{m/d}$ rather than the variance-fraction scale $m/d$.

For covariance shape, the relevant Hilbert space is the Gaussian score space,
and removing the unknown overall scale is an orthogonal projection of the
shape score onto the orthogonal complement of the nuisance scale score. The
resulting efficient Fisher information contracts at order $(m/d)^2$ for
diffuse shape directions (\cref{sec:rate3}).

\begin{table}[htbp]
\centering
\caption{Guide to the three recovery laws. Each law is stated in the units
of its recovery quantity.}
\label{tab:recovery-laws}
\begin{adjustbox}{max width=\textwidth}
\begin{tabular}{@{}lll@{}}
\toprule
Target & Recovery quantity & Recovery law \\
\midrule
Distance features, isotropic & recoverable variance fraction & $m/d$ \\
Neighbor rankings & excess pairwise agreement over chance
& $\pi^{-1}\sqrt{m/d}\,(1+o(1))$ \\
Diffuse covariance shape & Haar-mean efficient-information fraction
& $\dfrac{(m-1)(m+2)}{(d-1)(d+2)}\sim(m/d)^2$ \\
\bottomrule
\end{tabular}
\end{adjustbox}
\end{table}

The exponents differ because the three rows measure different objects.
Variance recovery squares an $L^2$ correlation, rank comparisons depend on a
sign probability and therefore on the correlation itself, and a covariance
perturbation is compressed on both sides by the random projector.

\section{Exact Solution of the Gaussian Distance Operator}
\label{sec:hgr}

For the beta--gamma channel of \cref{sec:channel}, the recovery operator
$\mathcal A$ of \cref{sec:operator} acts through the scalar pair
$(U,\mathcal T)$. This section computes its operator norm, then its complete
singular system, and derives the channel's mutual information and identities
for the distance.

\subsection{The operator norm}
\label{sec:hgr-norm}

The norm follows from a result on partial sums. Dembo, Kagan, and Shepp
\cite{dembo2001remarks} proved that the maximal correlation between
a partial sum of $r$ i.i.d. nondegenerate finite-variance variables and a
containing sum of $d$ such variables is $\sqrt{r/d}$. The chi-square
variables of \cref{sec:channel} satisfy these hypotheses. Ordinary linear
correlation already equals $\sqrt{r/d}$; the theorem says that nonlinear
transformations of the partial and total sums cannot improve it.

\begin{theorem}[Operator norm and maximal correlation]
\label{thm:hgr}
For the isotropic component and any rank-$r$ linear observation,
\[
\norm{\mathcal A}_{\mathrm{op}}
=\HGR(D;\obs)
=\sqrt{\frac rd}.
\]
Consequently,
$\Var(\E[f(D)\mid\obs])\le(r/d)\Var(f(D))$ for every $f\in L^2$, and
\[
\inf_g\E[(f(D)-g(\obs))^2]\ge(1-r/d)\Var(f(D)).
\]
\end{theorem}

The proof (\cref{app:proofs}) writes $\mathcal T$ as a sum of $d$
independent $\chi_1^2$ variables and $U$ as the partial sum of the first
$r$. The Dembo--Kagan--Shepp theorem gives
maximal correlation $\sqrt{r/d}$ for this pair. Because $U$ is sufficient for
$\mathcal T$---the rest of the sketch carries no further information about
the distance---the identity transfers to the full sketch. Thus thresholds, quantiles, clipped
distances, and every other square-integrable scalar feature obey the same
decoder-free ceiling. Mutual information complements this ceiling by
measuring total dependence \cite{cover2006elements}, but neither quantity
resolves recovery mode by mode; the singular system does.

\subsection{The singular spectrum}
\label{sec:rate1-iso}

The gamma law is the orthogonality measure for the generalized Laguerre
polynomials, and Griffiths \cite{griffiths1969canonical} diagonalized the gamma
conditional-expectation operator $h(\mathcal T)\mapsto\E[h(\mathcal T)\mid U]$
in that basis. For the beta--gamma pair of \cref{sec:channel}, this system
resolves recoverable variance one fluctuation component at a time.

\begin{theorem}[Laguerre spectrum]
\label{thm:laguerre}
The conditional-expectation operator
$f(\mathcal T)\mapsto\E[f(\mathcal T)\mid U]$ has generalized Laguerre
singular functions and singular values
\[
\ell_k=\left(\frac{(r/2)_k}{(d/2)_k}\right)^{1/2},
\qquad k=0,1,2,\ldots,
\]
where $(a)_k$ is the rising factorial \cite{griffiths1969canonical}. If
$f(\mathcal T)-\E f(\mathcal T)=\sum_{k\ge1}c_k\phi_k(\mathcal T)$ is the
orthonormal Laguerre expansion, then
\[
\Var(\E[f(\mathcal T)\mid\obs])=\sum_{k\ge1}\ell_k^2c_k^2.
\]
For fixed $k$, as $r\to\infty$ and $d-r\to\infty$,
\[
\ell_k=\left(\frac rd\right)^{k/2}
\left[1+O_k\!\left(\frac1r+\frac1d\right)\right].
\]
\end{theorem}

This is Griffiths's diagonalization specialized to the gamma parameters
$a=r/2$ and $a+b=d/2$; the proof, with that of \cref{thm:mi}, appears in
\cref{app:proofs}. The first nonconstant mode is the centered total
energy $\mathcal T-\E\mathcal T$ and has
$\ell_1=\sqrt{r/d}=\norm{\mathcal A}_{\mathrm{op}}$, so the operator norm of
\cref{thm:hgr} is attained by the centered distance itself. The second mode
is the orthogonalized quadratic fluctuation of $\mathcal T$. Thus $k$ indexes
increasingly nonlinear fluctuation components, not additional spatial
coordinates. At leading order, advancing from mode $k$ to mode $k+1$ costs an
additional factor of $\sqrt{r/d}$.

\subsection{Mutual information and identities for the distance}
\label{sec:hgr-identities}

The spectrum measures recoverable variance feature by feature. From the gamma
entropies of the observed and omitted squared norms, we also obtain the total
mutual information.

\begin{theorem}[Exact mutual information]
\label{thm:mi}
With $h_\Gamma(k)$ the entropy of a gamma variable of shape $k$ at any fixed
scale,
\[
I(D;\obs)=h_\Gamma\!\left(\frac d2\right)
-h_\Gamma\!\left(\frac{d-r}{2}\right)
\longrightarrow-\frac12\log(1-\alpha)
\]
if $r/d\to\alpha<1$ and $d-r\to\infty$. If $r/d\to0$, then
$I(D;\obs)=r/(2d)+O(r^2/d^2+1/d)$.
\end{theorem}

The proof (\cref{app:proofs}) evaluates the gamma entropies of the observed
and omitted squared norms; the limits follow from Stirling expansions.

For the rank-$r$ projector, define the rescaled projected distance by
\[
\widetilde D=\frac dr\norm{\Pi(X-X')}^2.
\]
The same decomposition gives three identities for the distance:
\[
\Corr(D,\widetilde D)=\sqrt{r/d},\qquad
\frac{\inf_f\E[(D-f(\widetilde D))^2]}{\Var(D)}=1-\frac rd,
\]
and
\[
\frac{\E[((\widetilde D-\E\widetilde D)-(D-\E D))^2]}{\Var(D)}
=\frac{d-r}{r}.
\]
Rescaling preserves the common baseline, but the fluctuation noise of
$\widetilde D$ has relative order $r^{-1/2}$ rather than the data's
$d^{-1/2}$.

For one pair under a rank-$m$ linear sketch,
\[
\HGR\bigl(D_{ij};(LX_i,LX_j)\bigr)=\sqrt{m/d}
\]
whereas the replacement cloud $Y^n$ of \cref{prop:zeroinfo} gives
\[
\HGR(D_{ij};Y^n)=0.
\]
Thus, the JL bound is compatible both with the maximal dependence
$\sqrt{m/d}$ attainable from a rank-$m$ linear sketch of one distance and
with zero dependence.

The isotropic operator is now solved. \Cref{sec:rate1} extends distance
recovery to general covariance, where the operator norm and the
recovery fraction for the distance can separate.

\section{Distance Recovery Beyond Isotropy}
\label{sec:rate1}

The isotropic solution is complete: \cref{thm:hgr,thm:laguerre} give the
operator norm and the full singular spectrum. Anisotropy separates two
questions that coincide for a flat spectrum: recovery of the distance value
and recovery of an arbitrary nonlinear feature.

\subsection{General covariance: the recovery fraction for the distance}
\label{sec:rate1-alpha}

\Cref{thm:hgr} is an isotropic statement about every feature. For general
covariance, recovery of the distance value itself admits an exact answer for
every spectrum and every rank-$m$ map. Because eigendirection $j$ contributes
$8\lambda_j^2$ to
$\Var(D)$, an optimal $m$-dimensional observation targets the largest squared
eigenvalues. We therefore define the \emph{recovery fraction for the distance
$D$} by
\begin{equation}
\alm(\Sigma)=\frac{\sum_{j\le m}\lambda_j^2}
{\sum_{j\le d}\lambda_j^2}\in(0,1].
\label{eq:alpham}
\end{equation}

We must separate two recovery questions. First, $\alm(\Sigma)$ is the largest
fraction of $\Var(D)$ recoverable by a rank-$m$ map.
Second, for a fixed observation, $\HGR^2(D;\obs)$ is the largest
recoverable fraction over \emph{all} square-integrable features of $D$. These quantities
coincide at $m/d$ in the isotropic case. Under anisotropy they can differ;
balanced spectral retention is a sufficient condition that restores
equality.

\begin{theorem}[Sharp rank-$m$ limit for distance recovery]
\label{thm:alpham}
For every rank-$m$ linear $L$ and every measurable estimator $g$,
\[
\frac{\Var\bigl(\E[D\mid\obs]\bigr)}{\Var(D)}\le\alm(\Sigma),
\qquad
\E\bigl[(D-g(\obs))^2\bigr]\ge 8\sum_{j>m}\lambda_j^2.
\]
The variance bound is attained when the row space of $L$ is the top-$m$
eigenspace of $\Sigma$; for this map, the mean-square bound is attained by
$g=\E[D\mid\obs]$. The same statements, with
$\Var(S)=12\tr(\Sigma^2)$, hold for the neighbor contrast
$S=\norm{X_0-X_1}^2-\norm{X_0-X_2}^2$ given
$(LX_0,LX_1,LX_2)$.
\end{theorem}

The proof (\cref{app:proofs}) combines the Gaussian conditional mean, a
projector identity, and eigenvalue interlacing for the compressed covariance.
The bounds are uniform: no nonlinear estimator of the distance value and no
$\Sigma$-aware choice of $L$ exceeds the squared-eigenvalue limit. The
mean-square bound also follows from the law of total variance: subtracting at
most $8\sum_{j\le m}\lambda_j^2$ recovered variance from
$\Var(D)=8\sum_{j\le d}\lambda_j^2$ leaves at least
$8\sum_{j>m}\lambda_j^2$.

\subsection{Anisotropy: divergence and reconciliation}
\label{sec:rate1-aniso}

The isotropic identity might suggest that $\sqrt{\alm(\Sigma)}$ bounds HGR
under anisotropy, but this bound fails in general. The mechanism is unequal
retention across independent spectral blocks:
observing one block completely while discarding another can let a nonlinear
transformation reweight their contributions. A two-dimensional counterexample
illustrates this mechanism.

\begin{theorem}[Anisotropic counterexample]
\label{thm:counterexample}
Let $D=2G_1^2+G_2^2$ for independent standard normals $G_1,G_2$, and observe
$U=2G_1^2$. Although $\alpha_1=4/5$, quadratic polynomial features already
give
\[
\HGR(D;U)\ge
\sqrt{\frac{246+2\sqrt{201}}{311}}
\approx0.9392
>\sqrt{4/5}.
\]
\end{theorem}

When $\HGR^2(D;\obs)>\alm(\Sigma)$, some nonlinear feature of $D$ is more
recoverable than the distance value itself. We call this gap
\emph{nonlinear excess}. Here $\sqrt{4/5}\approx0.8944$ is the linear
correlation with the distance. The proof (\cref{app:proofs}) evaluates the
canonical correlation of the quadratic features, which makes the displayed
value a lower bound on $\HGR$. For
$D_\lambda=\lambda G_1^2+G_2^2$ and $U_\lambda=\lambda G_1^2$,
\cref{fig:anisotropy} reports canonical-correlation lower bounds from nested
polynomial feature spaces. These values are lower bounds for a fixed observation,
not the unrestricted HGR value or an optimization over observations.

Nonlinear excess can occur but is not universal. A sufficient condition for
its absence follows by writing
\[
\Sigma=\sum_{g=1}^G\lambda_gP_g,
\qquad r_g=\rank P_g,
\]
where the $\lambda_g>0$ are distinct and $P_g$ projects onto the associated
eigenspace, or \emph{spectral block}. Let $\Pi$ satisfy
$\Pi\Sigma=\Sigma\Pi$, so it preserves every block, and define
\[
s_g=\rank(\Pi P_g),\qquad
\theta_g=\frac{s_g}{r_g},\qquad
\alpha_\Pi(\Sigma)
=\frac{\tr(\Pi\Sigma^2)}{\tr(\Sigma^2)}
=\frac{\sum_g\lambda_g^2s_g}{\sum_g\lambda_g^2r_g}.
\]
The Gaussian distance then splits across independent beta--gamma channels, one
per spectral block. Maximal correlation tensorizes over such independent pairs
\cite{witsenhausen1975sequences}: for nondegenerate $(Y_j,Z_j)$ independent
across $j$,
\[
\HGR\bigl((Y_j)_j;(Z_j)_j\bigr)=\max_j\HGR(Y_j;Z_j).
\]

\begin{theorem}[Spectral-block HGR bounds and the balanced regime]
\label{thm:block-hgr}
For independent $X,X'\sim\mathcal N(\mu,\Sigma)$, let
$D=\norm{X-X'}^2$ and $\obs=(\Pi X,\Pi X')$. Under
$\Pi\Sigma=\Sigma\Pi$,
\[
\sqrt{\alpha_\Pi(\Sigma)}
\le\HGR(D;\obs)
\le\sqrt{\max_g\theta_g}.
\]
If $s_g=\theta r_g$ for every nonzero spectral block, then
\[
\HGR(D;\obs)=\sqrt\theta=\sqrt{\alpha_\Pi(\Sigma)}.
\]
Consequently,
$\Var(\E[f(D)\mid\obs])\le\theta\Var(f(D))$ for every $f\in L^2(D)$,
and the centered distance paired with the centered retained squared norm
attains maximal correlation.
\end{theorem}

Witsenhausen tensorization bounds this product channel by its largest block
correlation, $\sqrt{\max_g\theta_g}$; a centered linear witness supplies the
lower bound (\cref{app:proofs} contains the full argument). Unequal
eigenvalues alone therefore do not break the balanced
identity. Unequal retention permits nonlinear excess but does not force it;
equal fractional retention makes the excess vanish.

The flat spectrum closes the gap between the two recovery quantities. If
$\Sigma=\lambda P$ has flat support of rank $r$, an observation retaining
rank $s$ within $\operatorname{range}(P)$ has $\HGR^2=s/r$; an optimal
rank-at-most-$m$ observation chooses its row space inside the support, takes
$s=\min\{m,r\}$, and attains $\HGR^2=\alm(\Sigma)$. More generally, for one
distance, $\alm(\Sigma)$ governs the distance value for every spectrum, and
$\alpha_\Pi(\Sigma)$ is also the ceiling for every feature under a balanced
aligned projection. Rankings compare several dependent distances sharing a
query and follow the square-root correlation scale associated with the same
rank ratio.

\section{Rankings Under Random Projection}
\label{sec:rate2}

\Cref{sec:hgr,sec:rate1} quantified the recovery of one distance. Rankings
instantiate the projection template of \cref{sec:targets} with the
shared-query contrast as target and the three projected points as
observation. Because rankings compare several dependent distances, we first
define the rank functional and identify the ranking scale before computing
its exact law.

For two tie-free score vectors $a,b\in\R^q$, Kendall's rank correlation is
\begin{equation}
\tau(a,b)=\binom q2^{-1}
\sum_{j<k}\operatorname{sgn}\!\left[(a_j-a_k)(b_j-b_k)\right].
\label{eq:bg-kendall}
\end{equation}
This statistic is the fraction of concordant pairs minus the fraction of discordant pairs
\cite{kendall1938new}. For independent continuous rankings, the expected
Kendall correlation is zero.

\subsection{Why \texorpdfstring{$\sqrt{m/d}$}{sqrt(m/d)}: sign statistics at the linear mode}
\label{sec:rate2-why}

We can identify the ranking scale before the exact calculation. For a query
$X_0$ and candidates $X_1,X_2$, write
\[
S=D_{01}-D_{02}=a^\top b,
\qquad
a=X_1-X_2,
\qquad
b=X_1+X_2-2X_0.
\]
For i.i.d. isotropic Gaussians, $a$ and $b$ are independent. Under the rescaled
rank-$m$ projector, the contrast is
$\widetilde S=(d/m)a^\top\Pi b$. Conditional on $a$, the pair
$(S,\widetilde S)$ is centered jointly Gaussian in $b$ with correlation
\[
\rho(a)=\frac{\norm{\Pi a}}{\norm{a}},
\qquad
\rho(a)^2\sim\operatorname{Beta}\!\left(\frac m2,\frac{d-m}{2}\right).
\]
The squared correlation has mean $m/d$.

For a centered jointly Gaussian pair, Sheppard
\cite{sheppard1899application} showed that the sign-agreement probability is
$\tfrac12+\pi^{-1}\arcsin\rho$. As $\rho\to0$, this agreement probability
exceeds chance by
$\rho/\pi+O(\rho^3)$. Thus the retained-energy fraction $m/d$ becomes a
ranking excess of order $\pi^{-1}\sqrt{m/d}$ on the correlation scale.
Averaging Sheppard's formula over the Beta-distributed conditional correlation
then yields the exact law.

\subsection{Exact pairwise order agreement}
\label{sec:pairwise-ranking}

Let $X_0,X_1,X_2\overset{\text{i.i.d.}}{\sim}
\mathcal N(\mu,\sigma^2I_d)$ independently of a Haar rank-$m$ projection,
where $0<m<d$. Let $D_{0j}$ and $\widetilde D_{0j}$ be the original and
rescaled projected squared distances.

\begin{theorem}[Exact pairwise order agreement]
\label{thm:arcsine}
The probability that the comparisons agree is
\[
p_{m,d}
=\Pr\bigl[(D_{01}-D_{02})(\widetilde D_{01}-\widetilde D_{02})>0\bigr]
=\frac12+\frac1\pi\E\!\left[\arcsin\sqrt B\right],
\qquad B\sim\mathrm{Beta}\!\left(\frac m2,\frac{d-m}{2}\right).
\]
Hence $p_{m,d}\le\tfrac12+\tfrac12\sqrt{m/d}$, and
\[
p_{m,d}=\frac12+\frac1\pi\sqrt{\frac md}
\left(1+O\!\left(\frac1m+\frac md\right)\right)
\]
when $m\to\infty$ and $m/d\to0$. For one query ranking any number of
candidates, the expected Kendall correlation is $\E\tau=2p_{m,d}-1$.
\end{theorem}

The proof (\cref{app:proofs}) conditions on $a$ in the preceding reduction
and averages Sheppard's formula over the Beta law. Thus, when $m=o(d)$, the
expected pairwise agreement is only $\Theta(\sqrt{m/d})$ above chance.

\begin{figure}[H]
\centering
\begin{minipage}[t]{.485\linewidth}
\vspace{0pt}
\centering
\includegraphics[width=\linewidth]{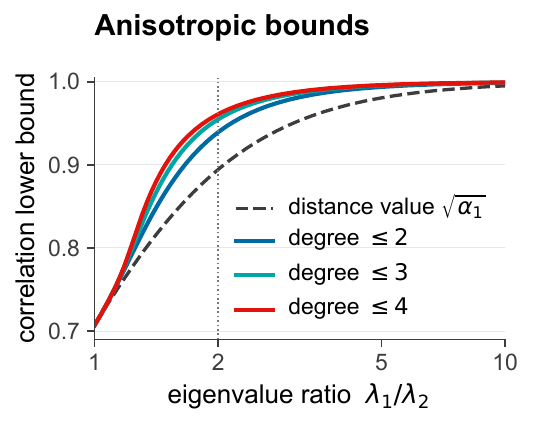}
\caption{Polynomial lower bounds for $D_\lambda\mapsto U_\lambda$. Higher
degrees reveal nonlinear dependence beyond $\sqrt{\alpha_1}$; the vertical
line marks the counterexample $\lambda=2$.}
\label{fig:anisotropy}
\end{minipage}\hfill
\begin{minipage}[t]{.485\linewidth}
\vspace{0pt}
\centering
\includegraphics[width=\linewidth]{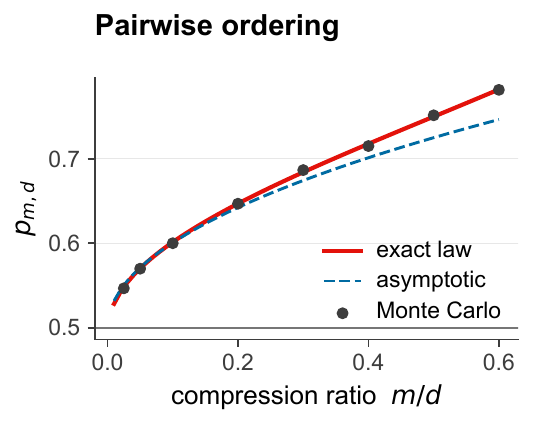}
\caption{Pairwise order agreement: the exact Beta--arcsine law,
its asymptote, and Monte Carlo simulations at $d=200$ with $120{,}000$
samples per marker.}
\label{fig:order}
\end{minipage}
\end{figure}

\subsection{Nearest-neighbor and top-\texorpdfstring{$k$}{k} overlap}
\label{sec:neighborhood-ranking}

Pairwise agreement does not determine whether the nearest candidate or the
full top-$k$ set is preserved. Let
$X_0,X_1,\ldots,X_q\overset{\text{i.i.d.}}{\sim}\mathcal N(0,I_d)$, with $X_0$
the query. Let $R\in\R^{m\times d}$ have orthonormal rows and be independent
of the sample; isotropy makes its orientation irrelevant. Let
$\widetilde{\mathcal N}_k$ and $\mathcal N_k$ be the projected and original
sets of $k$ nearest candidate indices.

Write
\[
U_j=\norm{R(X_j-X_0)}^2,
\qquad
V_j=\norm{(I-\Pi)(X_j-X_0)}^2,
\qquad \Pi=R^\top R.
\]
The projected ranking is determined by $U=(U_j)_j$, while the original ranking
is determined by $U+V$. For one coordinate, Gaussian fourth moments give, for
distinct candidates $j$ and $k$,
\[
\Var\!\left((x_j-x_0)^2\right)=8,
\qquad
\Cov\!\left((x_j-x_0)^2,(x_k-x_0)^2\right)=2.
\]
The centered and normalized vectors
$(U-2m\mathbf1)/\sqrt m$ and
$(V-2(d-m)\mathbf1)/\sqrt{d-m}$ therefore converge to independent Gaussian
score vectors with covariance $6I_q+2\mathbf1\mathbf1^\top$. The
rank-one term in this covariance represents a random shift shared by every
candidate. Removing this shift leaves
the ranks unchanged and gives correlation $\sqrt{m/d}$ between the retained
and full scores.

For independent standard normal vectors $G,H\in\R^q$, let
$\mathcal K_k(z)$ return the indices of the $k$ smallest coordinates and
define
\[
\mathcal R_{q,k}(\rho)=\frac1k\E\abs{
\mathcal K_k(G)\cap
\mathcal K_k\!\left(\rho G+\sqrt{1-\rho^2}H\right)}.
\]
The nearest-neighbor case is
$p_q(\rho)=\mathcal R_{q,1}(\rho)$, the probability that the two score vectors
have the same smallest coordinate. For
$G_\rho=\rho G+\sqrt{1-\rho^2}H$, define the minimum cells
\[
A_i=\{x\in\R^q:x_i=\min_jx_j\},
\qquad i=1,\ldots,q.
\]
Then
\begin{equation}
\sum_{i=1}^q\Pr(G\in A_i,\,G_\rho\in A_i)
=\Pr\!\left(\arg\min_iG_i=\arg\min_i(G_\rho)_i\right).
\label{eq:bg-partition-stability}
\end{equation}
The sum over all $q$ cells is essential; one cell has probability smaller by
a factor of $q$ by symmetry. Let $\phi_\rho$ be the standard bivariate-normal
density and let $\overline\Phi_\rho(x,y)=\Pr(Z_1>x,Z_2>y)$ be its joint survival
function. We also write $\overline\Phi(x)=\Pr(Z>x)$ for the univariate
standard-normal survival function. Conditioning on the common minimum gives
\begin{equation}
p_q(\rho)=q\iint_{\R^2}\phi_\rho(x,y)
\overline\Phi_\rho(x,y)^{q-1}\,dx\,dy.
\label{eq:pq}
\end{equation}
Each remaining Gaussian pair must exceed both conditioned values. By central
symmetry the same integral may be written with the lower joint CDF after
replacing the common minimum by a common maximum.

\begin{theorem}[Gaussian score limit for fixed candidate count]
\label{thm:knn-gaussian-limit}
Fix $q$ and $1\le k<q$. If $m\to\infty$, $d-m\to\infty$, and
$m/d\to\alpha\in[0,1]$, then
\[
\E\frac{\abs{\mathcal N_k\cap\widetilde{\mathcal N}_k}}{k}
\longrightarrow\mathcal R_{q,k}(\sqrt\alpha).
\]
For $k=1$,
\[
\Pr(\mathcal N_1=\widetilde{\mathcal N}_1)
=p_q\!\left(\sqrt{m/d}\right)
+O_q\!\left(m^{-1/2}+(d-m)^{-1/2}\right).
\]
In particular, $p_q(0)=1/q$ and $\mathcal R_{q,k}(0)=k/q$. Thus, for fixed
$q,k$ and $m/d\to0$, the probability that the nearest neighbors agree tends
to $1/q$, and the expected normalized top-$k$ overlap tends to $k/q$.
\end{theorem}

The proof (\cref{app:proofs}) applies Bentkus's multivariate Berry--Esseen
bound for convex sets \cite{bentkus2003dependence} to the independent
coordinate sums forming the projected and omitted scores; because $q$ is
fixed, summing over the $q$ disjoint choices of common minimizer gives the
stated $O_q(m^{-1/2}+(d-m)^{-1/2})$ transfer.

More generally, let $R_{r:q}$ be the rank of the second-coordinate value paired
with the $r$th ordered first-coordinate value. This paired value is the
\emph{concomitant} of the $r$th order statistic. Then
\[
\mathcal R_{q,k}(\rho)=\frac1k\sum_{r=1}^k\Pr(R_{r:q}\le k).
\]
The distribution of ranks in a finite sample is classical
\cite{david1977distribution}. We reduce neighborhoods after random projection
to these kernels at correlation $\sqrt{m/d}$, establish the finite-dimensional
transfer, and derive the top-$k$ limit.

Frieze and Jerrum derived the small-correlation expansion of the same Gaussian
common-argmax kernel \cite{frieze1997improved}. In common-minimum notation it
gives the following result; the proof, by differentiating \cref{eq:pq} at
$\rho=0$, appears in \cref{app:proofs}.

\begin{proposition}[Classical small-correlation expansion]
\label{prop:pq-small-rho}
For fixed $q\ge2$,
\[
p_q(\rho)=\frac1q+c_q\rho+O_q(\rho^2),
\qquad
c_q=q^2(q-1)\left[
\int_{-\infty}^{\infty}\phi(x)^2\overline\Phi(x)^{q-2}\,dx
\right]^2.
\]
In particular, $c_2=1/\pi$, recovering the small-correlation expansion of
Sheppard's formula. Hence
\[
\Pr(\mathcal N_1=\widetilde{\mathcal N}_1)
=\frac1q+c_q\sqrt{\frac md}
+O_q\!\left(\frac md+\frac1{\sqrt m}+\frac1{\sqrt{d-m}}\right).
\]
\end{proposition}

\subsection{The JL bound can hold while rankings collapse}
\label{sec:coexistence}

\Cref{sec:certification} showed that the JL bound alone forces no retained
geometry: an independent Gaussian replacement map can satisfy it. The same
separation also holds for a linear projection. Although the pairwise law
suggests that rankings degrade under thin projection, it does not show whether
one fixed map can satisfy the JL bound on the same sample while its rankings
collapse. We prove that these events coexist.

Let $A=\sqrt{d/m}\,R$, where $R$ has Haar-distributed orthonormal rows. For
distinct $x,y,z$, define
\[
s_A(x;y,z)=\operatorname{sgn}\!\left[
\bigl(\norm{x-y}^2-\norm{x-z}^2\bigr)
\bigl(\norm{A(x-y)}^2-\norm{A(x-z)}^2\bigr)
\right].
\]
For $n$ data points, the average Kendall correlation across queries is
\[
\overline\tau_n=\frac1n\sum_{i=1}^n
\binom{n-1}{2}^{-1}
\sum_{\substack{j<k\\j,k\ne i}}s_A(X_i;X_j,X_k),
\]
whose expectation involves the single-query constant
\[
\tau_{m,d}=\frac2\pi\E\arcsin\sqrt B,
\qquad B\sim\mathrm{Beta}\!\left(\frac m2,\frac{d-m}{2}\right):
\]
by \cref{thm:arcsine}, $\tau_{m,d}$ is the expected Kendall correlation for
one query under a rank-$m$ Haar projection.

\begin{theorem}[One map can satisfy the JL bound while rankings collapse]
\label{thm:jl-kendall}
Let $X_1,\ldots,X_n\overset{\text{i.i.d.}}{\sim}\mathcal N(0,I_d)$ independently
of $A$, with $n\ge3$, $0<m<d$, $0<\varepsilon<1$, and
$0<\delta,\eta<1$.
\begin{romannum}
\item \emph{Conditional concentration.} Conditionally on $A$,
$\E[\overline\tau_n\mid A]=\tau_{m,d}$ and, for $t>0$,
\[
\Pr\!\left(\abs{\overline\tau_n-\tau_{m,d}}>t\mid A\right)
\le2\exp\!\left(-\frac{\lfloor n/3\rfloor t^2}{2}\right).
\]
\item \emph{JL event.} A universal $C>0$ exists such that, if
$m\ge C\varepsilon^{-2}\log(n/\delta)$, then $A$ satisfies the
$\varepsilon$-JL bound on the sample with probability at least $1-\delta$.
\item \emph{Simultaneous conclusion.} With probability at least
$1-\delta-\eta$, the JL event in part (ii) and
\[
\overline\tau_n\le
\sqrt{\frac md}+
\sqrt{\frac{2\log(2/\eta)}{\lfloor n/3\rfloor}}
\]
hold simultaneously.
\end{romannum}
\end{theorem}

The proof (\cref{app:proofs}) symmetrizes the Kendall kernel and applies the
blocking inequality of Hoeffding \cite{hoeffding1963probability} for bounded
order-three U-statistics, partitioning the sample into disjoint triples to
restore independence and give conditional concentration; a union bound gives
the JL event. If
$n\to\infty$ and $\log n\ll m\ll d$, the same sampled map can satisfy the JL
bound for fixed $\varepsilon$ while
$\overline\tau_n\to0$ in probability.

\begin{remark}[Haar projection versus replacement map]
\label[remark]{rem:haar-vs-rep}
The normalized replacement distance has law $\chi_m^2/m$, exactly the
distortion law for one vector under an i.i.d. Gaussian JL matrix. For the Haar
map in \cref{thm:jl-kendall},
\[
\frac{\norm{Av}^2}{\norm{v}^2}
\overset d=\frac dmB,
\qquad
B\sim\mathrm{Beta}\!\left(\frac m2,\frac{d-m}{2}\right),
\qquad
\Var\!\left(\frac{\norm{Av}^2}{\norm{v}^2}\right)
=\frac{2(d-m)}{m(d+2)}.
\]
This variance agrees with $2/m$ only when $m/d\to0$. Even there, pairwise
distances are dependent in different ways: every projected pair uses the same
projection, whereas every replacement pair belongs to the same replacement
cloud. For one distance, the two constructions therefore occupy the opposite
endpoints recorded in \cref{sec:hgr-identities}: the maximal dependence
$\sqrt{m/d}$ attainable from a rank-$m$ linear sketch, and no dependence at
all.
\end{remark}

Our results show that ranking agreement follows the $\sqrt{m/d}$ correlation
scale, whereas the JL bound does not constrain it.
We next study clustering and estimation, which depend on how well a sketch
preserves the ability to distinguish nearby distributions.

\section{Covariance-Shape Information}
\label{sec:rate3}

Distance rankings concern single-sample observables. Clustering and estimation
instead ask how well projected data distinguish nearby distributions. We use
Fisher information to measure that local distinction and to remove overall
scale before measuring covariance shape.

For a regular parametric family with density $p_\theta$, the Fisher
information matrix is the second-moment matrix of the score,
\begin{equation}
\mathcal I(\theta)
=\E_\theta\!\left[
\nabla_\theta\log p_\theta(X)
\bigl(\nabla_\theta\log p_\theta(X)\bigr)^\top
\right].
\label{eq:bg-fisher}
\end{equation}
Under the usual regularity conditions it also equals
$-\E_\theta[\nabla_\theta^2\log p_\theta(X)]$. Fisher information gives the
local quadratic approximation
\[
D_{\mathrm{KL}}(P_\theta\,\|\,P_{\theta+\delta})
=\frac12\delta^\top\mathcal I(\theta)\delta+o(\norm{\delta}^2)
\]
\cite{vandervaart1998asymptotic}. It therefore measures local statistical
distinguishability; through the Cram\'er--Rao bound, its inverse also controls
the variance of regular estimators.

When a parameter of interest $\vartheta$ is accompanied by a nuisance
parameter $\zeta$, the efficient information removes score variation
explained by the nuisance direction. If the nuisance block is nonsingular,
the efficient information is the Schur complement
\begin{equation}
\mathcal I_{\mathrm{eff}}
=\mathcal I_{\vartheta\vartheta}
-\mathcal I_{\vartheta\zeta}
 \mathcal I_{\zeta\zeta}^{-1}
 \mathcal I_{\zeta\vartheta}.
\label{eq:bg-efficient-information}
\end{equation}
Equivalently, it is the information in the residual score after projecting
out the span of the nuisance score. This is the projection in score space
anticipated in \cref{sec:targets}: the target space is spanned by scores rather than by
features of a distance, and orthogonal projection again determines what the
observation retains.

\subsection{Fisher information under projection}

Let $R$ have Haar-distributed orthonormal rows and write $\Pi=R^\top R$. For
$X\sim\mathcal N(\mu,\sigma^2I_d)$ observed through this known projection,
the Fisher information for the mean is exactly
\begin{equation}
\mathcal I_\mu(RX)=\frac{1}{\sigma^2}\,\Pi.
\label{eq:fisher}
\end{equation}
It is the identity on the observed $m$-space and zero on the kernel. For a
scalar submodel $\mu=\theta v$, the retained fraction
$v^\top\Pi v/\norm{v}^2$ is
$\mathrm{Beta}(\tfrac m2,\tfrac{d-m}2)$-distributed with mean $m/d$, and
the Cram\'er--Rao variance inflation in direction $v$ is its reciprocal.

\subsection{Mean, scale, and shape contraction}

A local perturbation of a Gaussian component decomposes into a change of
mean, a scalar change of log-scale, and a traceless change of covariance
shape. To define shape independently of scale, consider
$\Sigma_{\zeta,\epsilon}=e^\zeta(I_d+\epsilon H)$ with $\tr H=0$. Here
$\zeta$ controls overall scale and $\epsilon$ the magnitude of the traceless
shape direction $H$. Treating $\zeta$ as a nuisance parameter projects the
shape score $RHR^\top$ away from the scale score, which is proportional to
$I_m$. The efficient shape score is therefore the traceless part of $RHR^\top$,
\[
RHR^\top-\frac{\tr(RHR^\top)}{m}I_m.
\]
Haar averaging gives a separate contraction factor for each perturbation
type.

\begin{table}[htbp]
\centering
\caption{Haar-averaged fractions of local Gaussian Fisher information,
equivalently of local Kullback--Leibler divergence. The shape fraction is
computed after eliminating unknown overall scale; unlike the log-scale
fraction, it varies across row spaces, and the table reports its Haar mean.}
\label{tab:local-information}
\begin{adjustbox}{max width=\textwidth}
\begin{tabular}{@{}lll@{}}
\toprule
Perturbation & Retained fraction & Dependence on row space \\
\midrule
Mean direction & $m/d$ & Beta-distributed around its mean \\
Scalar log-scale & $m/d$ & Deterministic \\
Diffuse traceless shape
& $\dfrac{(m-1)(m+2)}{(d-1)(d+2)}\sim(m/d)^2$
& Random; displayed fraction is the Haar mean \\
\bottomrule
\end{tabular}
\end{adjustbox}
\end{table}

\begin{theorem}[Local contraction taxonomy]
\label{thm:taxonomy}
A Haar rank-$m$ projection retains the three expected fractions summarized
in \cref{tab:local-information}; the shape row uses efficient Fisher
information after eliminating the nuisance scale $\zeta$. For any symmetric
traceless $H$, the exact moment identity is
\[
\E_R\norm{RHR^\top-\frac{\tr(RHR^\top)}{m}I_m}_F^2
=\frac{(m-1)(m+2)}{(d-1)(d+2)}\,\norm{H}_F^2.
\]
\end{theorem}

At the isotropic model, the unprojected efficient Fisher information for
$\epsilon$ is $\tfrac12\norm{H}_F^2$, while the projected information is one
half of the squared norm inside the expectation. The Haar moment identity
therefore gives the contraction factor for shape directly; the proof, based on
second-order Haar integration, appears in \cref{app:proofs}. A diffuse shape
perturbation distributes its squared Frobenius norm across many
eigendirections rather than concentrating it in a few spikes. Such a
perturbation loses information twice. Projection removes coordinates and also
reduces contrast within the retained coordinates. \Cref{fig:infogeo} compares
the exact constant with simulation.

For two well-separated Gaussian components with common isotropic covariance,
the leading logarithm of the Bayes error is proportional to the squared
distance between their means. A Haar projection multiplies that exponent by the
Beta-distributed factor in \cref{eq:fisher}, which has mean $m/d$ and
concentrates as both dimensions grow. Dasgupta
\cite{dasgupta1999learning} obtained logarithmic projection dimensions for
learning separated mixtures. Our local information rate is consistent with
that result but is not a complete clustering guarantee.

\subsection{Compression toward sphericity and the spike exception}

The local calculation describes infinitesimal changes. To connect it with a
global covariance, let $A=\sqrt{d/m}\,R$, define
$C_A=A\Sigma A^\top$, put $\bar\lambda=\tr(\Sigma)/d$, and set
$\Delta=\Sigma-\bar\lambda I$. Then
\[
C_A-\frac{\tr\Sigma}{m}I_m=\frac dmR\Delta R^\top.
\]
Standard bounds for random compression identify two dimensionless quantities
that control convergence to a scalar matrix in operator norm
\cite{dasgupta1999learning,dasgupta2000experiments}. The quantity
$\sqrt m\,\norm{\Delta}_F/\tr\Sigma$ controls the distributed anisotropic bulk:
the square-root factor reflects accumulation across many weak directions.
The quantity $m\,\norm{\Delta}_{\mathrm{op}}/\tr\Sigma$ controls the largest
individual direction, for which no such averaging occurs. Thus the compressed
covariance approaches a scalar matrix in regimes where both quantities vanish.
The resulting
rounding toward a scalar covariance is Dasgupta's sphericalization phenomenon.
Dasgupta
\cite{dasgupta1999learning,dasgupta2000experiments} established the phenomenon
for eccentric Gaussians, and Dasgupta, Hsu, and Verma
\cite{dasgupta2006concentration} extended it to more general distributions.

A sufficiently large spectral spike violates the operator-norm condition and
can remain visible after compression. Its detection threshold depends on the
population spectrum, sample size, and statistical task, so the rate derived
for diffuse shape should not be applied to finite-rank spikes without an
additional model.

The quadratic contraction completes the taxonomy: mean and scale contract at
$m/d$, diffuse traceless shape at order $(m/d)^2$, and a sufficiently large
spike escapes the diffuse rate. A known noiseless map can be row
orthonormalized without changing its information, whereas estimators that use the
output metric directly or invert noisy measurements remain sensitive to its
singular spectrum.

\section{General Maps: Rank, Spectrum, and Precision}
\label{sec:generalmaps}

Orthogonal projections lose information only by discarding directions: they
are isometries on their retained row spaces. A general linear map, by contrast, can also
distort the geometry it keeps. The relevant spectral summary depends on the
downstream task. \Cref{tab:map-taxonomy} previews the three cases.

\begin{table}[htbp]
\centering
\caption{Three downstream uses of a fixed known rank-$r$ map $T$, with
$M=T^\top T$ and nonzero singular values $s_1,\ldots,s_r$. The first two rows
take $Z\sim\mathcal N(0,I_d)$; the third takes
$y=Tx+\varepsilon\eta$ for independent $x\sim\mathcal N(0,I_d)$ and
$\eta\sim\mathcal N(0,I_m)$.}
\label{tab:map-taxonomy}
\begin{adjustbox}{max width=\textwidth}
\begin{tabular}{@{}lll@{}}
\toprule
Downstream use & Governing summary & Exact law \\
\midrule
Optimal nonlinear inference
& Rank $r$
& $\HGR(\norm{Z}^2;TZ)=\sqrt{r/d}$ \\
Unwhitened squared-norm estimate
& Trace ratio $r_2(M)$
& $\Corr(\norm{Z}^2,\norm{TZ}^2)=\sqrt{r_2(M)/d}$ \\
Inversion with additive noise
& Spectrum near zero
& $\MMSE(x\mid Tx+\varepsilon\eta,T)=
d-\sum_{i=1}^r s_i^2/(s_i^2+\varepsilon^2)$ \\
\bottomrule
\end{tabular}
\end{adjustbox}
\end{table}

\subsection{The rank--spectrum--precision trichotomy}

\begin{theorem}[Trichotomy]
\label{thm:trichotomy}
Let $Z\sim\mathcal N(0,I_d)$, let $T\in\R^{m\times d}$ be a fixed known map
of rank $r$ with nonzero singular values $s_i$, and write
$D=\norm{Z}^2$, $Q=\norm{TZ}^2$, and $M=T^\top T$.
\begin{romannum}
\item \emph{Best-possible inference is governed by rank alone:}
$\HGR(D;TZ)=\sqrt{r/d}$, so for every $f\in L^2$,
$\Var(\E[f(D)\mid TZ])\le(r/d)\Var(f(D))$.
\item \emph{The unwhitened squared-norm estimate is governed by the spectrum:}
\[
\Corr(D,Q)=\frac{\tr M}{\sqrt{d\,\tr(M^2)}}
=\sqrt{\frac{r_2(M)}{d}},
\qquad
r_2(M)=\frac{(\tr M)^2}{\tr(M^2)}.
\]
If $\kappa=s_{\max}/s_{\min}$ over the nonzero singular values, then
\[
r\frac{4\kappa^2}{(\kappa^2+1)^2}\le r_2(M)\le r.
\]
The lower bound is asymptotic to $4r/\kappa^2$ as $\kappa\to\infty$.
\item \emph{Recovery under additive noise is governed by the spectrum near
zero:} if $x\sim\mathcal N(0,I_d)$ and
$\eta\sim\mathcal N(0,I_m)$ are independent and
$y=Tx+\varepsilon\eta$, then
\[
\tr\Var(\E[x\mid y,T])=\sum_{i=1}^{r}\frac{s_i^2}{s_i^2+\varepsilon^2},
\qquad
\MMSE(x\mid y,T)=d-\sum_{i=1}^{r}\frac{s_i^2}{s_i^2+\varepsilon^2}.
\]
\end{romannum}
\end{theorem}

The proof (\cref{app:proofs}) treats the three parts separately.
Part (i) follows by orthonormalizing the rows of $T$ as in \cref{sec:channel}.
The singular value decomposition shows that observing $TZ$ is equivalent to observing
$r$ coordinates of a rotated standard Gaussian. With
$W=(TT^\top)^{\dagger/2}T$, one has
$\Corr(D,\norm{WZ}^2)=\sqrt{r/d}$. If $T$ is invertible, then
$D=(TZ)^\top(TT^\top)^{-1}(TZ)$ exactly.

Part (ii) analyzes $Q=\norm{TZ}^2$ without row-orthonormalizing $T$. We bound
its correlation with $D$ using
the Kantorovich inequality: for positive $x_i\in[x_{\min},x_{\max}]$,
\[
\frac{(\sum_{i=1}^r x_i)^2}{r\sum_{i=1}^r x_i^2}
\ge
\frac{4x_{\min}x_{\max}}{(x_{\min}+x_{\max})^2}.
\]
Taking $x_i=s_i^2$ and
$x_{\max}/x_{\min}=\kappa^2$ yields the displayed trace-ratio bound.

In part (iii), we ask whether inverting $T$ remains stable under additive noise.
The factor $s_i^2/(s_i^2+\varepsilon^2)$ is the variance recovered in singular
direction $i$ by the Gaussian posterior mean, or equivalently the Wiener-filter
gain. When $s_i\ll\varepsilon$, the recovered fraction in direction $i$ is
negligible even if the map is invertible. Invertibility alone does not imply that the
output Euclidean metric preserves distances or that noisy inversion is
stable.

\subsection{Gaussian-map examples}

Gaussian maps isolate reversible spectral distortion from rank loss. A square
i.i.d. Gaussian map is full rank, yet its unwhitened centered-distance
correlation tends to $1/\sqrt2$; row orthonormalization restores the Haar value.
\Cref{fig:harmonic} shows the one-stage comparison.

\begin{figure}[H]
\centering
\begin{minipage}[t]{.485\linewidth}
\vspace{0pt}
\centering
\includegraphics[width=\linewidth]{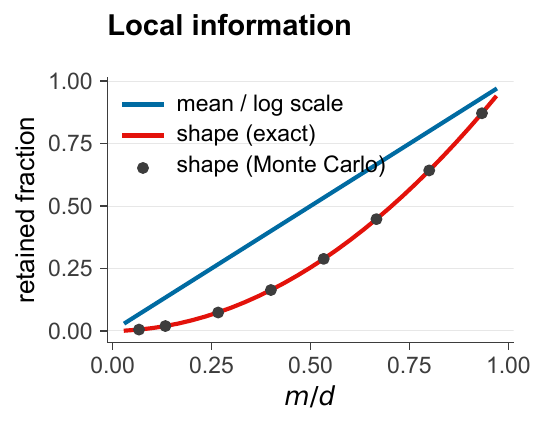}
\caption{Local information contraction. Mean and log-scale directions retain
$m/d$; traceless shape retains
$(m-1)(m+2)/((d-1)(d+2))$. Markers use $d=60$ and $1{,}500$ Haar frames per
$m$.}
\label{fig:infogeo}
\end{minipage}\hfill
\begin{minipage}[t]{.485\linewidth}
\vspace{0pt}
\centering
\includegraphics[width=\linewidth]{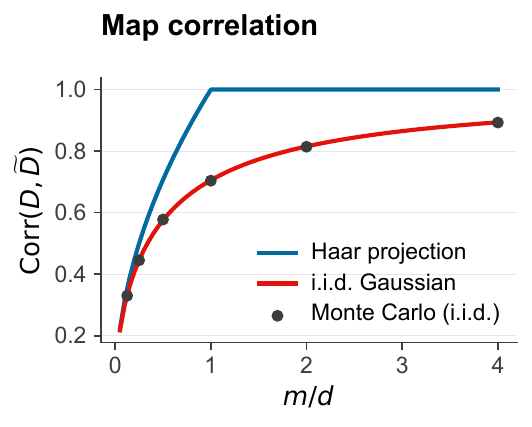}
\caption{Correlation between original and mapped centered squared distances
after one Haar or i.i.d.\ Gaussian stage. Markers use the chi-square
representation at $d=200$, with $150{,}000$ samples per $m$.}
\label{fig:harmonic}
\end{minipage}
\end{figure}

A single map exposes the distinction among rank loss, reversible spectral
distortion, and instability under noise. We next ask how multiple independent
sketches change the available rank.

\section{Ensembles}
\label{sec:ensembles}

We now consider multiple independent sketches, which reduce the variability
of a single sketch. By stacking them, we obtain their information ceiling; by
averaging them, we obtain a simple decoder whose finite-size gap can be
computed exactly.

\begin{theorem}[Ensemble ceiling and averaging]
\label{thm:ensemble}
Let $A_1,\dots,A_k\in\R^{m\times d}$ be sketches applied to the same data.
\begin{romannum}
\item \emph{Converse.} Stacking gives one linear map of rank at most
$p=\min\{km,d\}$. For any joint decoder, the recoverable centered-distance
variance is at most $\alpha_p(\Sigma)$. In the isotropic case,
$\HGR\le\sqrt{p/d}$, with equality for independent Gaussian blocks and the
optimal joint decoder.
\item \emph{Explicit averaging decoder.} If the $A_i$ are independent
Gaussian maps and we average the $k$ rescaled squared distances they produce,
then the correlation, jointly over the maps and data, is
\[
\sqrt{\frac{km}{km+d+2}}.
\]
The ratio of this correlation to the HGR ceiling is
\[
\gamma_{k,m,d}
=\sqrt{\frac{km\,d}{(km+d+2)\min\{km,d\}}}
=
\begin{cases}
\sqrt{d/(km+d+2)}, & km\le d,\\
\sqrt{km/(km+d+2)}, & km\ge d.
\end{cases}
\]
\end{romannum}
\end{theorem}

For the averaging decoder, the proof (\cref{app:proofs}) conditions on the
standardized difference to obtain an estimate of the form $\norm{Z}^2C$,
where
$C\sim\chi^2_{km}/(km)$ is independent of $Z$. The resulting correlation is
the exact effect of an independent multiplicative chi-square fluctuation,
including the finite $d+2$ correction: averaging $k$ sketches leaves the
distance estimate intact and shrinks this fluctuation by a factor $\sqrt k$
in standard deviation.

Independent sketches add rank, but they never exceed the rank-$km$ ceiling.
The converse in part (i) is the stacked-rank special case of \cref{thm:alpham};
the averaging identity is the single-stage correlation of the concatenated
$km\times d$ Gaussian map (\cref{app:proofs}).
Simple averaging has the correct scaling and approaches this ceiling when
$km\ll d$ or $km\gg d$; its largest constant-factor gap occurs near
$km\approx d$, where $\gamma_{k,m,d}\approx1/\sqrt2$. Before $km$ reaches
$d$, added sketches expand the available rank; afterward, they only reduce
distortion in the averaging decoder. Correcting the joint metric when
$km\le d$, or inverting the stacked map when $km\ge d$, removes this
reversible spectral distortion and reaches the isotropic ceiling.

These identities are exact population statements. We verify representative
finite-size claims in \cref{sec:verification}.

\section{Numerical Validation}
\label{sec:verification}

The preceding results include exact population identities, asymptotic limits,
and probability statements at finite $d$ and $m$. We use numerical validation for
two purposes: to check independent implementations of the closed-form
expressions and to measure how closely experiments at finite $d$ and $m$ follow
their limiting predictions. We use exact Gaussian moments for polynomial
identities, deterministic quadrature for integral formulas, and Monte Carlo
sampling for stochastic claims at finite $d$ and $m$. We report one standard
error for Monte Carlo estimates and Wilson $95\%$ intervals for frequencies of
joint events. These calculations support the formulas and their interpretation
at finite dimensions; they do not replace the proofs.

The complete reproducibility materials are available at
\begin{center}
\url{https://github.com/piyush314/random-projection-geometry}.
\end{center}
The reproducibility guide in \cref{sec:supp-reproducibility} records the tested
version and explains how to set up, execute, and generate the artifacts.

The independent Gaussian replacement experiment tests whether the JL bound
distinguishes a genuine shared projection from a cloud that is
sampled independently of the data. It compares a normalized Haar projection,
an i.i.d. Gaussian map, and the independent Gaussian replacement map at the same
$(n,d,m,\varepsilon)$ values. For $n=100$, $d=8192$, $m=1024$, and
$\varepsilon=0.2$, all $30$ Haar trials and all $30$ Gaussian-map trials
satisfied the JL bound, as did $28$ of $30$ replacement trials. At
$m/d=1/4$, the Pearson correlation of centered squared distances in
\cref{fig:zero-info-jl} was $0.55$ for the Haar projection and $0.01$ for the
replacement map; both displayed maps satisfied the
same JL bound. \Cref{fig:zero-info-jl} reports the manuscript
view, while the reproducibility repository records the full ranking,
nearest-neighbor, calibration, and disjoint-pair diagnostics.

Every calculation based on exact moments or quadrature in \cref{tab:verify} agrees with
its reference value to the displayed precision, and every Monte Carlo estimate
agrees within two standard errors. Two rows require separate interpretation.
The quarter-circle entry is the limiting mean of a posterior variance fraction.
The entry for $r_2$ is a deterministic equivalent formed from a ratio of
moments, not an identity for a finite sample.

\begin{table}[htbp]
\centering
\caption{Numerical validation summary. Monte Carlo (MC) estimates are estimate $\pm$ one standard error; exact-moment and quadrature checks have no sampling error. Both isotropic mean-square error rows are normalized by $\Var(D)$.}
\label{tab:verify}
\begin{adjustbox}{max width=\textwidth}
\footnotesize
\setlength{\tabcolsep}{4pt}
\begin{tabular}{@{}llll@{}}
\toprule
Claim & Reference value & Computed value & Method \\
\midrule
Isotropic $\Corr(D,\widetilde D)$ & 0.2236 & 0.2228\,$\pm$\,0.0015 & MC, independent chi-square split, $N=400{,}000$ \\
Normalized minimum mean-square error & 0.9500 & 0.9519\,$\pm$\,0.0022 & MC, conditional mean, $N=400{,}000$ \\
Centered estimation MSE / $\operatorname{Var}(D)$ & 19.000 & 19.100\,$\pm$\,0.053 & MC, rescaled projection, $N=400{,}000$ \\
Quadratic canonical correlation at $\lambda{=}2$ & 0.939239 & 0.939239 & exact Gaussian moments \\
Beta--arcsine $p_{5,50}$ & 0.59829 & 0.59907\,$\pm$\,0.00077 & MC, conditional Gaussian, $N=400{,}000$ \\
\shortstack[l]{Diffuse traceless covariance-shape\\contraction ($8/40$)} & 0.04274 & 0.04297\,$\pm$\,0.00018 & MC, Haar frames, $N=3{,}000$ \\
One Gaussian stage ($200\to200$) & 0.70535 & 0.70442\,$\pm$\,0.00080 & MC, chi-square stage, $N=400{,}000$ \\
Gaussian chain ($300\to150\to100$) & 0.4058 & 0.4050\,$\pm$\,0.0013 & MC, independent chi-square stages, $N=400{,}000$ \\
Wishart $\E\tr(M^2)$ & 243.00 & 243.14\,$\pm$\,0.38 & MC, Gaussian maps, $N=3{,}000$ \\
\shortstack[l]{Ratio-of-moments proxy vs.\\Monte Carlo mean of $r_2$} & 14.8148 & 14.8413\,$\pm$\,0.0070 & MC; proxy is not an identity, $N=3{,}000$ \\
Quarter-circle mean posterior-variance fraction ($\varepsilon{=}0.2$) & 0.180998 & 0.180998 & deterministic quadrature \\
Ensemble average ($km{=}100,d{=}400$) & 0.44632 & 0.44618\,$\pm$\,0.00080 & MC, stacked Gaussian map, $N=1{,}000{,}000$
\\
\bottomrule
\end{tabular}
\end{adjustbox}
\end{table}

\Cref{tab:new-verify} validates the results for balanced blocks, the Gaussian
common minimum, nearest neighbors, and JL--Kendall coexistence. The
calculation for balanced blocks uses exact moments to check, without Monte
Carlo error, the degree-two lower bound on canonical correlation and hence on
nonlinear excess. The nearest-neighbor row compares a simulation at finite
$d$ and $m$, $0.137020$, with its Gaussian score limit,
$0.138325$, at $m=20$. Here the $m^{-1/2}$ Berry--Esseen remainder can exceed
the first-order signal, so simulation provides the check at finite $d$ and $m$.
The final row tests the coexistence
phenomenon of \cref{thm:jl-kendall}. Its joint rate, $0.990$, is the empirical
frequency with maximum sample distortion at most $15\%$ and mean Kendall
correlation magnitude at most $0.1$.

\begin{table}[htbp]
\centering
\caption{Independent checks of the incorporated results. Monte Carlo
uncertainty is one standard error; the JL--Kendall coexistence event has Wilson
$95\%$ interval $[0.9710,0.9966]$.}
\label{tab:new-verify}
\begin{adjustbox}{max width=\textwidth}
\begin{tabular}{@{}lll@{}}
\toprule
Claim and setting & Prediction & Computation \\
\midrule
Balanced blocks, $\theta=0.25$ & $\HGR=0.5$ & degree-two CCA lower bound $0.500000000000$ \\
$p_8(\sqrt{20/10^4})$ & integral in \cref{eq:pq} & $0.138324998$ \\
kNN, $(d,m,q)=(10^4,20,8)$ & Gaussian score limit $0.138325$ & $0.137020\pm0.000544$ \\
JL--Kendall coexistence, $(n,d,m)=(100,10^6,2000)$ & $\E\tau=0.028476$ & $0.028996\pm0.001254$; joint rate $0.990$ \\
\bottomrule
\end{tabular}
\end{adjustbox}
\end{table}

We interpret two checks in more detail because they reveal finite-sample
behavior that asymptotic notation can obscure. First, we calculate exact
Gaussian moments at $\lambda=2$ to reproduce the quadratic canonical
correlation in \cref{thm:counterexample}. We then sweep the polynomial
canonical correlation of $D_\lambda=\lambda G_1^2+G_2^2$ and
$U_\lambda=\lambda G_1^2$ to produce \cref{fig:anisotropy}. Degree one agrees
with $\sqrt{\alpha_1}$, while nested degree-two through degree-four spaces
yield larger lower bounds away from isotropy. These values concern one
observation and finite polynomial classes; they neither evaluate unrestricted
HGR nor optimize the observed subspace.

Second, the rows for Gaussian stages in \cref{tab:verify} distinguish an exact
correlation averaged over map and data randomness from a deterministic
equivalent for a typical map. We compare one- and two-stage simulations with
the finite products in \cref{thm:harmonic}, including their $+2$ corrections.
A separate Wishart experiment checks
$\E\tr(M^2)=d(m+d+1)/m$. By contrast, $md/(m+d+1)$ is a ratio of trace moments,
not $\E[r_2]$; the adjacent row records the finite-sample difference rather
than treating the proxy as an identity.

Having checked the population formulas and their finite-dimensional
approximations, we turn to their practical implications.

\section{Discussion: Implications for Practice}
\label{sec:discussion}

The distinct preservation laws imply different responses in practice.
Reranking repairs a coarse neighbor stage, task alignment changes what
information matters, and independent sketches trade added rank against
distortion in an unwhitened estimate.

\paragraph{Candidate generation and reranking}
Locality-sensitive hashing and compressed indexes can support a coarse candidate stage \cite{indyk1998approximate,charikar2002similarity}. \Cref{thm:knn-gaussian-limit} explains why evaluating a short list at full precision can still matter: at fixed $q$, the projected top choice approaches an independent uniform choice when $m/d\to0$. This theorem does not prescribe an engineering threshold because real margins, candidate counts, and data distributions differ from the Gaussian benchmark. It does replace a pairwise heuristic with a complete ranking law for a fixed number of candidates.

\paragraph{Objectives governed by the baseline rather than fluctuations}
Subspace embeddings, sketched regression, and randomized low-rank approximation
target residual norms or dominant subspaces rather than the ordering of nearly
equal within-cloud distances \cite{halko2011finding}. Randomized singular value
decomposition is designed to retain spectral spikes, not diffuse anisotropy.
Its success is therefore compatible with the recovery limits derived here.

\paragraph{Coding and task alignment}
Product quantization changes the resource from output dimension to a finite bit rate \cite{jegou2011product}. For a Gaussian source, \emph{reverse water-filling} assigns rate only to covariance modes above a distortion threshold and allocates more rate to larger variances \cite{cover2006elements}. Learned rotations can then redistribute distortion across coordinates \cite{ge2013optimized}. Supervised nested-representation training, often called \emph{Matryoshka training}, takes a different route. It trains each leading coordinate block to remain predictive, so label-relevant information appears early in the representation \cite{kusupati2022matryoshka}. Both approaches add assumptions absent from an oblivious rank-$m$ projection. The relevant resource is task information, not total within-cloud distance recovery.

\paragraph{Cannings--Samworth ensembles}
Cannings and Samworth \cite{cannings2017random} aggregate classifiers trained
on multiple random projections. Our population calculation isolates the
geometric benefit of the parallel sketches: before $km=d$, more sketches
raise both rank and averaging accuracy; afterward, they improve only the
chosen estimate (\cref{thm:ensemble}). Related cluster ensembles use the same
parallel strategy \cite{fern2003random}. Task gains still depend on margins
and downstream learning.

\paragraph{Limits of transfer}
These practical responses rely on exact Gaussian laws whose transfer beyond the benchmark requires care.
Gaussianity makes the chi-square channel and score covariance exact. The relative-error criterion in \cref{prop:jl-margin} is distribution-free, and the proof mechanisms suggest extensions under invariance principles or elliptical models, but the constants need not survive. Aligned anisotropy is now covered by \cref{thm:block-hgr}; noncommuting maps remain open. Fixed-$q$ neighbor overlap is solved for isotropic Gaussian data and an independent orthogonal projection, while growing neighborhoods, hubness, density estimates, and non-Gaussian universality require a theory for distances that share a query.

\section{Research Directions}
\label{sec:conjectures}

The operator toolkit developed here---conditional-expectation channels,
singular systems, and sharp Gaussian limits---extends to settings the scalar
isotropic channel does not directly reach. This section formulates the
resulting challenges as precise problems, each anchored to a theorem or
conjecture the framework already produces.

\begin{openproblem}[General anisotropic maximal correlation]
\label[openproblem]{op:aniso}
For $D=\sum_i\lambda_iG_i^2$ observed through a rank-$m$ map, determine $\HGR(D;\obs)$ and the optimizing observation. We solve the problem for commuting projectors with balanced spectral retention in \cref{thm:block-hgr}, which also bounds every commuting block allocation. \Cref{thm:counterexample} shows that an unbalanced allocation can exceed $\sqrt{\alm}$. For a noncommuting map, the observed quadratic form no longer splits along covariance eigenspaces, preventing direct use of the blockwise beta--gamma reduction or its tensorization. The unrestricted problem couples a weighted-gamma conditional-expectation operator with an optimization over row spaces.
\end{openproblem}

\begin{openproblem}[Uniform concomitant asymptotics and growing top-$k$ sets]
\label[openproblem]{op:growq}
For fixed $\rho\in(0,1)$, Gaussian plurality stability measures the probability
that two correlated Gaussian score vectors select the same extremal coordinate.
The corresponding literature gives the classical asymptotic
\cite{khot2007optimal,deklerk2004approximate}
\[
p_q(\rho)\sim
\frac{\Gamma(1/(1+\rho))^2}{\sqrt{1-\rho^2}}
(q-1)^{-(1-\rho)/(1+\rho)}
\bigl(4\pi\log(q-1)\bigr)^{-\rho/(1+\rho)}.
\]
When the correlation varies with $q$, existing uniform Gaussian-tail control
and Borell's bound imply \cite{khot2007optimal}
\[
\limsup_{q\to\infty}q\,p_q(\rho_q)\le e^{2\lambda}
\qquad\text{when}\qquad \rho_q\log q\to\lambda\in[0,\infty).
\]
A matching lower bound, ideally with a uniform error estimate, remains open in
this triangular regime, where $q\to\infty$ and
$\rho_q\log q\to\lambda$. The case $\lambda=0$ already follows from the
upper bound and $p_q(\rho)\ge1/q$ for $\rho\ge0$. More generally, determine
overlap when $q$, $k$, and $\rho^{-1}$ grow jointly. The goal is to establish
a sharp triangular-regime and top-$k$ extension of the Gaussian stability
kernel.
Deterministic ordinal capacity and the corresponding Gaussian-cloud
conjecture are recorded in
\cref{thm:ordinal,conj:ordinal-gaussian}.
\end{openproblem}

\begin{conjecture}[Thin-compression ($m/d\to0$) spike phase diagram]
\label[conjecture]{conj:spikes}
Consider $m/d\to0$ with centered covariance variation spread across many eigendirections and with no dominant eigenvalue. First, identify conditions under which the normalized compressed traceless bulk has a universal limiting spectrum. Second, determine when a finite-rank population spike becomes distinguishable from that bulk. A semicircle transition is plausible only under additional assumptions. Under \emph{free compression}---the free-probability model for compression by a generic random projection---the compressed spectral law still depends on the population law. For example, compressing $\sigma^2I+\lambda vv^\top$ produces a scalar matrix plus rank one, with no population-level random bulk. A detection threshold must therefore depend jointly on projection dimension, sample size, and covariance, as in deformed random-matrix models, which study random spectral bulks perturbed by low-rank structure \cite{baik2005phase}.
\end{conjecture}

\begin{openproblem}[Growing-neighborhood statistics]
\label[openproblem]{op:multivariate}
Determine the leading components of density, hubness-magnitude, and intrinsic-dimension estimators when the number of distances sharing a query grows. Here \emph{hubness} is the tendency of a few points to be nearest neighbors of many queries \cite{radovanovic2010hubs}; hub magnitude counts how often this occurs, whereas hub identity records which points become hubs. Although \cref{thm:knn-gaussian-limit} resolves fixed-$q$ top-$k$ overlap, the single-distance Laguerre decomposition in \cref{thm:laguerre} does not tensorize across a shared query. A complete theory of these statistics must therefore cover non-Gaussian score limits and distinguish changes in hub magnitude from changes in hub identity.
\end{openproblem}

\begin{openproblem}[Supervised task-recovery limit]
\label[openproblem]{op:supervised}
Define $\alm^{\mathrm{task}}$ as the Fisher information about a task parameter retained under the best rank-$m$ linear map. We seek to determine when supervised training can make $\alm^{\mathrm{task}}$ close to one at small $m$ while total-covariance recovery remains small. This includes objectives that train nested prefixes of a representation to remain predictive, such as Matryoshka-style training \cite{kusupati2022matryoshka}. Characterize the resulting trade-off with nuisance geometry.
\end{openproblem}

\section{Conclusion}
\label{sec:conclusion}

Random projections do not preserve a single, undifferentiated notion of
geometry. On the Gaussian benchmark, the retention laws reviewed in
\cref{tab:recovery-laws} share one parameter, the rank ratio $m/d$, yet act
on different scales: a variance fraction for distance features, a
correlation for their signs, and a squared ratio for diffuse shape. Under
general covariance, $\alm(\Sigma)$ replaces $m/d$ for the distance value,
and sign statistics inherit the associated correlation scale, so
neighborhoods with a fixed number of candidates approach chance as
$m/d\to0$.

The differences among these scales have direct consequences. Because shape
contracts faster than mean and scale, compression sphericalizes distributed
anisotropy while concentrated spikes can remain visible. Because rankings
depend on the correlation scale, a projection can satisfy the JL bound on a
finite sample while its mean ranking correlation vanishes---and the
replacement-map construction shows the bound is compatible with output
entirely independent of the data. The JL bound alone therefore does not
force data-dependent geometry.

For a general known map, rank governs optimal decoding, spectral spread
governs an unwhitened squared-norm estimate, and the smallest singular values
govern noisy inversion. Multiple independent sketches can raise the combined
rank, but only up to the ambient dimension. These exact Gaussian laws show
where reranking, metric correction, or additional sketches are necessary.
The same conditional-expectation toolkit extends beyond the Gaussian
benchmark; \cref{sec:conjectures} develops the resulting research directions,
from anisotropic maximal correlation and growing neighborhoods to compressed
covariance phase diagrams and supervised task recovery.

\section*{Acknowledgments}

This work was supported by the Applied Mathematics program of the U.S.
Department of Energy, Office of Science, Office of Advanced Scientific
Computing Research (ASCR), through the SPARSITUTE Mathematical Multifaceted
Integrated Capability Center (MMICC), \emph{A Mathematical Institute for
Sparse Computations in Science and Engineering}
(\href{https://sparsitute.lbl.gov/}{sparsitute.lbl.gov}). We thank David
Rabson, ASCR program manager; Ramakrishnan Kannan, acting director of the
SPARSITUTE MMICC; and Michael L. Parks, the Oak Ridge National Laboratory
point of contact for ASCR Applied Mathematics, for their leadership and
support.

\clearpage
\appendix
\section{Gaussian Chains and Hard-Edge Precision}
\label{sec:supp-transforms}

Hanin and Nica \cite{hanin2020products} showed that the harmonic sum
$H_L=\sum_{\ell=1}^L n_\ell^{-1}$ of inverse layer widths governs norm
fluctuations in long random-matrix products applied to a fixed vector. For a
random Gaussian input, we add the input-fluctuation term
$n_0^{-1}=d^{-1}$. The spectrum law in \cref{thm:trichotomy}(ii) then gives the
following exact finite-chain identity.

\begin{theorem}[Gaussian chains]
\label{thm:harmonic}
Let $P_L=G_L\cdots G_1$ be a chain of independent Gaussian stages
$G_\ell\in\R^{n_\ell\times n_{\ell-1}}$ with entries $N(0,1/n_\ell)$, where
$n_0=d$. Apply the chain to $Z\sim\mathcal N(0,2\sigma^2I_d)$. With correlation
taken jointly over $Z$ and the stages,
\[
\Corr^2\bigl(\norm{Z}^2,\norm{P_LZ}^2\bigr)
=\frac{2/d}{\prod_{\ell=0}^{L}(1+2/n_\ell)-1}
\]
at every finite size. A single stage of width $m$ therefore gives
$\Corr=\sqrt{m/(m+d+2)}$. For a typical sampled map,
\[
r_2(P_L^\top P_L)
=(1+o_{\mathbb P}(1))
\left(\sum_{\ell=0}^{L}\frac1{n_\ell}\right)^{-1}.
\]
For one stage, the deterministic equivalent is $(1/m+1/d)^{-1}$. The finite
ratio of moments is $md/(m+d+1)$ because $\E\tr M=d$ and
$\E\tr M^2=d(m+d+1)/m$.
\end{theorem}

Conditional on its input, each stage multiplies squared norm by an independent
$\chi^2_{n_\ell}/n_\ell$ factor, so relative second moments multiply. The finite
identity concerns the joint law of stages and input, whereas the $r_2$ statement
concerns a typical fixed map asymptotically. The verification in
\cref{sec:verification} distinguishes the exact identity from its deterministic
equivalent.

The polar decomposition factors one Gaussian stage into a Haar-distributed
row-orthonormal factor and an independent Wishart factor. The first imposes the
rank bottleneck; the second adds a multiplicative $\chi_d^2/d$ norm
fluctuation. Their product has the full $\chi_m^2/m$ fluctuation. Thus a square
Gaussian stage is full rank but has unwhitened correlation
$\sqrt{d/(2d+2)}\to1/\sqrt2$. In contrast, a square orthogonal stage does not
distort Euclidean distance. For $L$ square Gaussian stages of width $n$, the
Hanin--Nica harmonic sum contributes $L/n$ and the random input contributes
$1/n$. Hence $r_2\sim n/(L+1)$ and correlation tends to $1/\sqrt{L+1}$.

Sequential composition differs from parallel concatenation. Concatenating $k$
independent $m$-row sketches gives $1/r_2\sim1/d+1/(km)$, as in
\cref{thm:ensemble}. The chain equivalent is specific to the i.i.d. Gaussian
stages considered above and to ensembles with the same limiting trace moments;
structured transforms can have other spectral laws. Moreover, $r_2$ contains
no singular-vector information, so an anisotropic diagnostic also needs an
alignment quantity such as
\[
\rho_\Sigma(T)^2
=\frac{[\tr(M\Sigma^2)]^2}
{\tr(\Sigma^2)\tr(M\Sigma M\Sigma)}.
\]

\subsection{Precision and the hard edge}

Marchenko and Pastur \cite{marchenko1967distribution} derived the limiting
law that, for a normalized square Gaussian map, gives the quarter-circle
singular-value density
\[
\rho(s)=\frac1\pi\sqrt{4-s^2},\qquad0\le s\le2.
\]
The origin is the \emph{hard edge}: singular values are constrained to be
nonnegative, and the limiting density remains nonzero there. Under the
additive-noise model in \cref{thm:trichotomy}(iii), the fraction of directions
with signal-to-noise ratio below one is approximately
$(2/\pi)\varepsilon$, while the average posterior loss has the exact limit
\[
\frac1d\MMSE(x\mid y,T)
\longrightarrow
\psi_{\mathrm{edge}}(\varepsilon)
=\frac\varepsilon2\left(\sqrt{\varepsilon^2+4}-\varepsilon\right)
=\varepsilon-\frac{\varepsilon^2}{2}+O(\varepsilon^3).
\]
Edelman \cite{edelman1988eigenvalues} showed that
$s_{\min}=\Theta_{\mathbb P}(1/d)$ after normalization. Thus, if
$\varepsilon\approx2^{-b}$, recovering every direction requires
$b\gtrsim\log_2d$. The average unrecovered fraction is instead approximately
$2^{-b}$, independent of $d$. These statements calibrate additive Gaussian
noise; floating-point roundoff is not itself isotropic additive noise.

\section{Proofs}
\label{app:proofs}
\sloppy

Throughout, let $Z=X-X'$, so $Z\sim\mathcal N(0,2\Sigma)$ and $D=\norm{Z}^2$.

\subsection{Proof of \texorpdfstring{\cref{prop:zero-info-sharpness}}{the sharp dimension order for independent Gaussian replacement maps}}

Let $r=\lfloor n/2\rfloor$ as in \cref{prop:zero-info-sharpness}. For the
disjoint pairs $(2\ell-1,2\ell)$, define
\[
R_\ell
=\frac{D_{2\ell-1,2\ell}^{\mathrm{rep}}}
       {D_{2\ell-1,2\ell}}.
\]
The variables $R_1,\ldots,R_r$ are i.i.d., and
\[
R_\ell\overset d=
\frac{\chi_m^2/m}{\chi_d^2/d}\sim F_{m,d}.
\]
Let $U=\chi_m^2/m$ and $V=\chi_d^2/d$ be independent. Chen and Rubin
\cite{chen1986bounds} showed that the median of a gamma variable is smaller
than its mean, so
$\Pr(V\le1)\ge1/2$. Moreover, the Zhang--Zhou lower bound for the
$\chi^2$ upper-tail probability
\cite[Cor.~3]{zhang2020nonasymptotic} gives, uniformly for
$0<\varepsilon\le1$,
\[
\Pr(U\ge1+\varepsilon)\ge c_0e^{-C_0m\varepsilon^2}.
\]
The event on the left together with $V\le1$ implies $U/V\ge1+\varepsilon$.
Thus, if
\[
p^{\mathrm{rep}}_{m,d}(\varepsilon)
=\Pr\!\left(F_{m,d}\notin[1-\varepsilon,1+\varepsilon]\right),
\]
then $p^{\mathrm{rep}}_{m,d}(\varepsilon)\ge
ce^{-Cm\varepsilon^2}$. The JL event implies that every disjoint-pair
ratio lies in the prescribed interval, hence
\[
\Pr(\mathcal E_\varepsilon^{\mathrm{rep}})
\le\bigl(1-p^{\mathrm{rep}}_{m,d}(\varepsilon)\bigr)^r
\le\exp\!\left(-cr e^{-Cm\varepsilon^2}\right).
\]
If $\Pr(\mathcal E_\varepsilon^{\mathrm{rep}})\ge1-\delta$, taking logarithms and
rearranging proves the displayed necessary condition in
\cref{prop:zero-info-sharpness}. Finally,
$-\log(1-\delta)\le2\delta$ for $\delta\le1/2$, which gives the stated
order. \qed

\subsection{Moment matching for the replacement cloud}

In the setting of \cref{prop:zeroinfo}, let
$N_p=\binom n2$ and $a_\varepsilon=\varepsilon/(2+\varepsilon)$. The Gaussian
replacement cloud matches the population mean squared distance but inflates its variance:
$\Var(D_{ij}^{\mathrm{rep}})=8\sigma^4d^2/m$, whereas
$\Var(D_{ij})=8\sigma^4d$. The replacement map satisfies the JL bound through
concentration rather than moment matching. More generally, if two distance families have common
mean $\mu_D$ and variances bounded by $v_X$ and $v_Y$, Chebyshev's inequality
and a union bound give
\[
\Pr(\text{failure})
\le\frac{N_p(v_X+v_Y)}{a_\varepsilon^2\mu_D^2}.
\]
For the two Gaussian clouds,
\[
\frac{v_X}{\mu_D^2}=\frac2d,
\qquad
\frac{v_Y}{\mu_D^2}=\frac2m,
\qquad
\Pr(\text{failure})
\le\frac{2N_p}{a_\varepsilon^2}\left(\frac1d+\frac1m\right).
\]
Chebyshev control alone therefore requires
$\min\{m,d\}=\Omega(n^2/(\varepsilon^2\delta))$. The logarithmic dimension in
\cref{prop:zeroinfo} comes from sub-exponential concentration, not merely from
the first two moments. Exact matching of those moments imposes a different
dimensional obstruction.

\begin{proposition}[Exact first-two-moment matching requires half dimension]
\label{prop:replacement-two-moments}
Let $Y,Y'$ be i.i.d. random vectors in $\R^m$ with finite fourth moments and
covariance $C$. Then
\[
\E\norm{Y-Y'}^2=2\tr C
\]
and
\[
\Var(\norm{Y-Y'}^2)
=2\Var(\norm{Y-\E Y}^2)+4\tr(C^2).
\]
If these moments equal $2\sigma^2d$ and $8\sigma^4d$, respectively, then
$m\ge d/2$. This lower bound is sharp when $d$ is even: for $m=d/2$ and
$U\sim\operatorname{Unif}(S^{m-1})$, the replacement
$Y=\sigma\sqrt d\,U$ matches both moments exactly.
\end{proposition}

Centering $Y$ does not change pairwise distances, so assume $\E Y=0$.
Expanding
\[
\norm{Y-Y'}^2=\norm{Y}^2+\norm{Y'}^2-2Y^\top Y'
\]
and using independence gives
\[
\E\norm{Y-Y'}^2=2\tr C,
\qquad
\Var(\norm{Y-Y'}^2)
=2\Var(\norm{Y}^2)+4\E(Y^\top Y')^2.
\]
The remaining expectation is $\tr(C^2)$. Moment matching therefore implies
$\tr C=\sigma^2d$ and
\[
8\sigma^4d
\ge4\tr(C^2)
\ge\frac{4(\tr C)^2}{m}
=\frac{4\sigma^4d^2}{m},
\]
where the middle inequality is Cauchy--Schwarz on the eigenvalues of $C$.
Thus $m\ge d/2$.

For the claimed extremizer, $\norm{Y}^2=\sigma^2d$ is deterministic and
$\Cov(Y)=(\sigma^2d/m)I_m$. Hence
\[
\E\norm{Y-Y'}^2=2\sigma^2d,
\qquad
\Var(\norm{Y-Y'}^2)
=4\tr(\Cov(Y)^2)
=\frac{4\sigma^4d^2}{m}
=8\sigma^4d.
\]
These matching moments prove sharpness. Thus exact first-two-moment matching
requires a linear share of the ambient dimension, even though
the replacement map in \cref{prop:zeroinfo} satisfies the JL bound in
logarithmic dimension without matching the variance. \qed

\subsection{Proof of \texorpdfstring{\cref{thm:hgr}}{the maximal-correlation theorem}}

After row orthonormalization, the sketch is informationally equivalent to $(\Pi X,\Pi X')$ with $\Pi$ a rank-$r$ projector. Set $Z=(X-X')/(\sqrt2\sigma)$, $\mathcal T=\norm{Z}^2$, and $U=\norm{\Pi Z}^2$. The Gaussian midpoint is independent of $Z$, and the conditional law of $\mathcal T$ given $\Pi Z$ depends on the projected vector only through $U$. Hence $U\sim\chi^2_r$, $V=\mathcal T-U\sim\chi^2_{d-r}$, $U\perp V$, and $\mathcal T\perp\obs\mid U$. Sufficiency then gives $\HGR(D;\obs)=\HGR(\mathcal T;U)$.

Write $\mathcal T=\sum_{i=1}^d W_i$ with $W_i$ i.i.d.\ $\chi^2_1$ and, by rotation, $U=\sum_{i\le r}W_i$. The partial-sum theorem of Dembo, Kagan, and Shepp \cite{dembo2001remarks} yields $\HGR=\sqrt{r/d}$, proving \cref{thm:hgr}. The variance and MMSE corollaries follow because $\Var(\E[f\mid\obs])/\Var(f)\le\HGR^2$. \qed

\subsection{Proof of \texorpdfstring{\cref{thm:laguerre,thm:mi}}{the Laguerre-spectrum and information theorems}}

Retain the variables $U$, $V$, and $\mathcal T$ from the preceding proof. For
\cref{thm:laguerre}, $(U,\mathcal T)$ is the classical gamma pair with $U/\mathcal T\sim\mathrm{Beta}(\tfrac r2,\tfrac{d-r}2)$ independent of $\mathcal T$. Griffiths \cite{griffiths1969canonical} identifies the canonical (singular) system of the induced conditional-expectation operator as generalized Laguerre polynomials with singular values $\ell_k=[(r/2)_k/(d/2)_k]^{1/2}$, and Parseval's identity yields the variance decomposition. The stated asymptotics follow from $\log \ell_k^2=\sum_{j<k}\log\frac{r/2+j}{d/2+j}$.

Finally, for \cref{thm:mi}, $I(D;\obs)=I(\mathcal T;U)=h(\mathcal T)-h(\mathcal T\mid U)=h(\mathcal T)-h(V)$ by independence, which is the stated gamma-entropy difference; the limits follow from Stirling expansions of $h_\Gamma$. \qed

\subsection{Proof of \texorpdfstring{\cref{thm:alpham}}{the theorem on recovering the distance}}

The conditional mean $\widehat X=\E[X\mid LX]$ is Gaussian with covariance $B=\Sigma L^{\top}(L\Sigma L^{\top})^{\dagger}L\Sigma=\Sigma^{1/2}P\Sigma^{1/2}$, where the pseudoinverse identity implies that $P=\Sigma^{1/2}L^{\top}(L\Sigma L^{\top})^{\dagger}L\Sigma^{1/2}$ is idempotent of rank at most $m$. Conditioning on $\obs$ gives $X-X'\sim\mathcal N(\widehat X-\widehat X',\,2(\Sigma-B))$, so
\[
\E[D\mid\obs]=\norm{\widehat X-\widehat X'}^2+2\tr(\Sigma-B).
\]
Because $\widehat X-\widehat X'\sim\mathcal N(0,2B)$,
\[
\Var(\E[D\mid\obs])=8\tr(B^2).
\]

Since $\tr(B^2)=\tr\bigl((P\Sigma P)^2\bigr)$, Poincar\'e separation implies that the eigenvalues of the compression $P\Sigma P$ are dominated pairwise by $\lambda_1,\dots,\lambda_m$. Hence $\tr(B^2)\le\sum_{j\le m}\lambda_j^2$, with equality when $\mathrm{range}(P)$ is the top-$m$ eigenspace. We derive the minimum mean-square error (MMSE) claim from the law of total variance:
\[
\inf_g\E[(D-g(\obs))^2]
=\E\Var(D\mid\obs)
=\Var(D)-\Var(\E[D\mid\obs])
\ge 8\sum_{j>m}\lambda_j^2.
\]

For the contrast, set $a=X_1-X_2$, $b=X_1+X_2-2X_0$, and $c=X_0+X_1+X_2$. Covariance computation shows that $a,b,c$ are mutually independent, $S=a^{\top}b$, and $\Var(S)=\tr(6\Sigma\cdot2\Sigma)=12\tr(\Sigma^2)$. The observations are a linear bijection of $(La,Lb,Lc)$, so $a,b$ remain conditionally independent given $\obs$. Hence $\E[S\mid\obs]=\widehat a^{\top}\widehat b$ and $\Var(\widehat a^{\top}\widehat b)=\tr(2B\cdot6B)=12\tr(B^2)$, yielding the same ratio. \qed

\subsection{Proof of \texorpdfstring{\cref{thm:counterexample}}{the anisotropic counterexample}}

Let $A=G_1^2$ and $B=G_2^2$, whose moments are $1,3,15,105$. The covariances among $\{D,D^2\}$ and $\{U,U^2\}$ are polynomials in these moments:
\[
\begin{aligned}
\Var D&=10, & \Cov(D,D^2)&=132, & \Var D^2&=2240,\\
\Var U&=8, & \Cov(U,U^2)&=96, & \Var U^2&=1536,\\
\Cov(D,U)&=8, & \Cov(D,U^2)&=96, & \Cov(D^2,U)&=112, & \Cov(D^2,U^2)&=1728.
\end{aligned}
\]
The largest eigenvalue of $\Sigma_{DD}^{-1}\Sigma_{DU}\Sigma_{UU}^{-1}\Sigma_{UD}$ is $(246+2\sqrt{201})/311\approx0.882170$, and its square root exceeds $\sqrt{4/5}$. Restricting to quadratic polynomials makes this value a lower bound on $\HGR$. \qed

\subsection{Proof of \texorpdfstring{\cref{thm:block-hgr}}{the spectral-block theorem}}

Remove the independent Gaussian midpoint and whiten inside each spectral block. Then
\[
D=2\sum_{g=1}^G\lambda_gT_g,
\qquad T_g=U_g+V_g,
\]
where the block pairs are independent and
\[
U_g\sim\chi^2_{s_g},\qquad
V_g\sim\chi^2_{r_g-s_g},\qquad U_g\perp V_g.
\]
For estimating $D$, the vector $(U_g)_g$ is sufficient for the full sketch. The beta--gamma channel $T_g\mapsto U_g$ has nontrivial maximal correlation $\sqrt{s_g/r_g}$. Witsenhausen \cite{witsenhausen1975sequences} proved the tensorization theorem; together with data processing, it gives
\[
\HGR(D;\obs)\le\max_g\sqrt{s_g/r_g}.
\]

For the lower bound, take $G(U)=2\sum_g\lambda_g(U_g-s_g)$. The omitted weighted sum is independent of $G(U)$, so
\[
\Cov(D,G(U))=\Var(G(U))=8\sum_g\lambda_g^2s_g,
\qquad
\Var(D)=8\sum_g\lambda_g^2r_g.
\]
Thus $\Corr(D,G(U))=\sqrt{\alpha_\Pi(\Sigma)}$. If every $s_g/r_g$ equals $\theta$, the bounds coincide. The flat-support statement is the one-block specialization. \qed

\subsection{Proof of \texorpdfstring{\cref{thm:arcsine}}{the Beta--arcsine law}}

With $a=X_1-X_2$ and $b=X_1+X_2-2X_0$ independent isotropic Gaussians, $D_{01}-D_{02}=a^{\top}b$ and $\widetilde D_{01}-\widetilde D_{02}=\tfrac dm a^{\top}\Pi b$. Conditional on $a$, these differences are centered jointly Gaussian in $b$ with correlation $\sqrt{a^{\top}\Pi a/a^{\top}a}$, whose square is $\mathrm{Beta}(\tfrac m2,\tfrac{d-m}2)$ by rotation invariance. Sheppard \cite{sheppard1899application} showed that the sign-agreement probability is $\tfrac12+\tfrac1\pi\arcsin\rho$; averaging over $a$ yields the law. Applying $\arcsin\sqrt B\le\tfrac\pi2\sqrt B$ and $\E\sqrt B\le\sqrt{m/d}$ gives the bound, whereas concentration of $B$ at $m/d$ gives the expansion. Kendall's $\tau$ over one query's ranking averages pairwise sign agreements, each with marginal probability $p_{m,d}$ by exchangeability; hence $\E[\tau]=2p_{m,d}-1$. \qed

\subsection{Proof of \texorpdfstring{\cref{thm:knn-gaussian-limit,prop:pq-small-rho}}{the Gaussian score limit for fixed candidate count}}

By rotational invariance, take $R$ to select the first $m$ coordinates and put $r=d-m$. Split $X_j=(Y_j,Z_j)$ and define
\[
U_j=\norm{Y_j-Y_0}^2,\qquad V_j=\norm{Z_j-Z_0}^2.
\]
The projected ranking is that of $(U_j)_j$, the original ranking is that of $(U_j+V_j)_j$, and the vectors $U$ and $V$ are independent.

For one coordinate, set
\[
W=((X_1-X_0)^2-2,\ldots,(X_q-X_0)^2-2).
\]
Direct Gaussian moments give $\Cov(W)=6I_q+2\mathbf1\mathbf1^\top=:C_q$. Independently,
\[
\frac{U-2m\mathbf1}{\sqrt m}\Rightarrow G^{(1)},
\qquad
\frac{V-2r\mathbf1}{\sqrt r}\Rightarrow G^{(2)},
\]
where $G^{(1)},G^{(2)}\sim\mathcal N(0,C_q)$. We may write
\[
G^{(1)}=\sqrt6G+\sqrt2\xi\mathbf1,
\qquad
G^{(2)}=\sqrt6H+\sqrt2\eta\mathbf1.
\]
The common shifts do not affect rankings. The limiting projected scores are therefore $G$, and the limiting original scores are
$\sqrt\alpha G+\sqrt{1-\alpha}H$. Ties have probability zero, so the continuous mapping theorem proves the top-$k$ limit.

For $k=1$, the event that one fixed index minimizes both rankings is an intersection of linear halfspaces in the $2q$ normalized score coordinates. Applying Bentkus's multivariate Berry--Esseen bound for convex sets \cite{bentkus2003dependence} to the $m$-coordinate sum and then to the independent $r$-coordinate sum gives error $O_q(m^{-1/2}+r^{-1/2})$. Summing over the $q$ disjoint common minimizers gives the stated rate. Conditioning on the score pair of the common minimizer yields \cref{eq:pq}: all other pairs must lie above it, so the joint survival function appears. This proves \cref{thm:knn-gaussian-limit}.

For \cref{prop:pq-small-rho}, differentiate \cref{eq:pq}. At $\rho=0$,
\[
\partial_\rho\phi_\rho(x,y)=xy\phi(x)\phi(y),
\qquad
\partial_\rho\overline\Phi_\rho(x,y)=\phi(x)\phi(y).
\]
Set $a_q=\int\phi(x)^2\overline\Phi(x)^{q-2}\,dx$. Integration by parts gives
\[
\int x\phi(x)\overline\Phi(x)^{q-1}\,dx=-(q-1)a_q.
\]
The first differentiated term contains the square of this integral, and the second contributes $q(q-1)a_q^2$. Hence $p_q'(0)=q^2(q-1)a_q^2$. A Taylor expansion proves the proposition. \qed

\subsection{Proof of \texorpdfstring{\cref{thm:jl-kendall}}{the JL--Kendall coexistence theorem}}

For a fixed rank-$m$ projector and an isotropic Gaussian triple, \cref{thm:arcsine} gives
$\E s_A(X_1;X_2,X_3)=\tau_{m,d}$, independently of the orientation of $A$. Define the symmetric kernel
\[
h_A(x,y,z)=\frac13\bigl(s_A(x;y,z)+s_A(y;x,z)+s_A(z;x,y)\bigr).
\]
A counting identity gives
\[
\overline\tau_n=\binom n3^{-1}\sum_{1\le i<j<k\le n}h_A(X_i,X_j,X_k).
\]
The kernel takes values in $[-1,1]$. Hoeffding
\cite{hoeffding1963probability} proved the blocking inequality for bounded
order-three U-statistics; applying it yields
\[
\Pr\!\left(\abs{\overline\tau_n-\tau_{m,d}}>t\mid A\right)
\le2\exp\!\left(-\frac{\lfloor n/3\rfloor t^2}{2}\right).
\]

For each fixed nonzero $v$, the ratio $\norm{Av}^2/\norm{v}^2$ has a scaled Beta law and satisfies
\[
\Pr\!\left(\abs{\frac{\norm{Av}^2}{\norm{v}^2}-1}>\varepsilon\right)
\le2e^{-c\varepsilon^2m}.
\]
Conditioning on the data and taking a union bound over the $\binom n2$ differences proves the stated JL event. Finally, \cref{thm:arcsine} gives $\tau_{m,d}\le\sqrt{m/d}$. Substituting the stated $t$ into the U-statistic bound and applying a union bound with the JL event completes the proof. \qed

\subsection{Proof of \texorpdfstring{\cref{thm:taxonomy}}{the shape taxonomy}}

\emph{Mean direction.} The Kullback--Leibler divergence $D_{\mathrm{KL}}$ between shifted isotropic Gaussians is $\norm{\delta}^2/2\sigma^2$ before projection and $\norm{R\delta}^2/2\sigma^2$ afterward. The ratio $\norm{R\delta}^2/\norm{\delta}^2$ has distribution $\mathrm{Beta}(\tfrac m2,\tfrac{d-m}2)$ and mean $m/d$.

\emph{Scalar log-scale.} The divergence $D_{\mathrm{KL}}=\tfrac d2(r-1-\log r)$ maps deterministically to $\tfrac m2(r-1-\log r)$ because $R(\sigma^2I)R^{\top}=\sigma^2I_m$.

\emph{Diffuse traceless covariance shape.} Consider $\Sigma_{\zeta,\varepsilon}=e^\zeta(I+\varepsilon H)$ with $H$ symmetric traceless and with $\zeta$ an unknown nuisance scale. At the isotropic model, the covariance-tangent Fisher inner product is $\langle A,B\rangle=\tfrac12\tr(AB)$. The original shape tangent $H$ is orthogonal to the scale tangent $I_d$. After projection, removing the component of $RHR^\top$ parallel to the nuisance tangent $I_m$ gives the efficient shape tangent
$RHR^{\top}-\tfrac{\tr(RHR^{\top})}mI_m$. For $\tr H=0$, second-order Haar integration gives
\[
\E\tr\!\left((RHR^\top)^2\right)
=\frac{m((m+1)d-2)}{d(d-1)(d+2)}\,\norm{H}_F^2,
\qquad
\E\!\left[\tr(RHR^\top)^2\right]
=\frac{2m(d-m)}{d(d-1)(d+2)}\,\norm{H}_F^2.
\]
Subtracting $1/m$ times the second identity from the first yields the traceless moment in \cref{thm:taxonomy}. Dividing the projected efficient Fisher information by $\tfrac12\norm{H}_F^2$ gives the exact ratio. Equation~\eqref{eq:fisher} follows from the standard Gaussian Fisher-information calculation with known $R$. \qed

\subsection{Proof of \texorpdfstring{\cref{thm:trichotomy,thm:harmonic,thm:ensemble}}{the transform and ensemble results}}

\emph{Rank.} By the singular value decomposition, $TZ$ is an invertible function of $(s_iZ_i')_{i\le r}$ for a rotated standard Gaussian $Z'$. Because each $s_i$ is nonzero, observing $TZ$ is informationally equivalent to observing $r$ coordinates of $Z'$, so \cref{thm:hgr} applies with $r$ in place of $m$.

\emph{Unwhitened squared-norm estimate.} With $M=T^{\top}T$, Gaussian quadratic-form identities give $\Cov(D,Q)=2\tr(M)$ and $\Var(Q)=2\tr(M^2)$ for $Z\sim\mathcal N(0,I_d)$. Hence $\Corr=\tr M/\sqrt{d\,\tr(M^2)}$. The Kantorovich inequality applied to the positive eigenvalues $s_i^2$ gives the stated two-sided bound, with the extremizer placing mass $\approx1/(\kappa^2+1)$ at $s_{\max}$. Whitening makes $W^{\top}W$ the row-space projector, so $r_2(W^{\top}W)=r$.

\emph{Additive-noise recovery.} Gaussian conditioning gives recovered variance $s_i^2/(s_i^2+\varepsilon^2)$ and posterior variance $\varepsilon^2/(s_i^2+\varepsilon^2)$ in each right-singular direction. Averaging the posterior variance over the quarter-circle law yields $\psi_{\mathrm{edge}}(\varepsilon)=\tfrac\varepsilon2(\sqrt{\varepsilon^2+4}-\varepsilon)$.

\emph{Finite joint identity.} Conditional on its input $y$, a Gaussian stage with entries $N(0,1/n_\ell)$ outputs $G_\ell y\sim\mathcal N(0,\norm{y}^2/n_\ell\,I_{n_\ell})$. Thus $\norm{G_\ell y}^2=\norm{y}^2\cdot\chi^2_{n_\ell}/n_\ell$, with the $\chi^2$ factor independent of everything prior. Hence $\E[\norm{P_LZ}^2\mid Z]=\norm{Z}^2$, and second moments multiply:
\[
\E\norm{P_LZ}^4=\E\norm{Z}^4\prod_{\ell\ge1}(1+2/n_\ell).
\]
With $\E\norm{Z}^4=(\E\norm{Z}^2)^2(1+2/d)$, these moments give
\[
\Var(\norm{P_LZ}^2)
=(\E\norm{Z}^2)^2\left[\prod_{\ell\ge0}(1+2/n_\ell)-1\right],
\qquad n_0=d.
\]
Moreover, $\Cov(\norm{Z}^2,\norm{P_LZ}^2)=\Var(\norm{Z}^2)=(\E\norm{Z}^2)^2\cdot2/d$. Their ratio is the stated identity, and $L=1$ gives $m/(m+d+2)$.

\emph{Typical-map deterministic equivalent.} The identities $\E\tr M=d$ and $\E\tr M^2=d(m+d+1)/m$ give the ratio of moments $md/(m+d+1)$. Concentration of both traces yields $r_2=(1+o_{\mathbb P}(1))\,md/(m+d)$. For multiple layers, the recursion
\[
\E[\tr(P_\ell^{\top}P_\ell)^2\mid P_{\ell-1}]
=\left(1+\frac1{n_\ell}\right)\tr(P_{\ell-1}^{\top}P_{\ell-1})^2
+\frac1{n_\ell}\left(\tr P_{\ell-1}^{\top}P_{\ell-1}\right)^2
\]
adds $1/n_\ell$ per layer at leading order. Rank-one maps ($r_2\equiv1$) show that this equivalent is not a finite identity.

\emph{Polar decomposition.} Write
\[
G/\sqrt m=(GG^{\top}/d)^{1/2}\cdot\sqrt{d/m}R,
\]
where $R$ is Haar and independent of the Wishart factor. For a fixed $x$, the Haar stage gives $\tfrac dm B$ with $B\sim\mathrm{Beta}(\tfrac m2,\tfrac{d-m}2)$. Along the projected direction, the square factor contributes $u^{\top}(GG^{\top}/d)u\sim\chi^2_d/d$. The beta--gamma algebra then gives $\tfrac dm B\cdot\chi^2_d/d\overset{d}=\chi^2_m/m$, which is the full map's noise.

\emph{Ensembles.} Stacking is a linear map of rank $\le\min\{km,d\}$, so \cref{thm:alpham} for general $\Sigma$ and \cref{thm:trichotomy}(i) apply with that rank. Independent Gaussian blocks concatenate into one Gaussian matrix whose row space is Haar, giving equality after whitening. The average of per-sketch estimates equals the unwhitened squared-norm estimate from the concatenated $km\times d$ Gaussian map, with entries $N(0,1/(km))$ after scaling. The single-stage identity therefore gives correlation $\sqrt{km/(km+d+2)}$ over both map and data randomness. Dividing this correlation by the HGR ceiling $\sqrt{\min\{km,d\}/d}$ gives the ratio stated in \cref{thm:ensemble}. \qed

\section{A Nonasymptotic Neighborhood Coupling}
\label{sec:coupling}

The Gaussian score limit for fixed candidate count in the main article assumes
that both coordinate blocks grow. In contrast, \cref{thm:knn-gap} covers fixed
$m$ and gives an explicit, conservative dependence on the candidate count.

\begin{theorem}[Nonasymptotic chance-neighborhood coupling]
\label{thm:knn-gap}
In the setting of \cref{thm:knn-gaussian-limit}, let $r=d-m\ge2$ and $a=\log(2q)$. Define
\[
c_r=\frac1{\sqrt{12\pi}}
\frac{\Gamma((r-1)/2)}{\Gamma(r/2)}
\asymp r^{-1/2}
\]
and
\[
\Delta_{q,m,r}=\min\left\{1,
\binom q2c_r\left(12\sqrt{ma}+16a\right)\right\}.
\]
There are independent uniformly random $k$-subsets $S_U,S_V$ of $\{1,\ldots,q\}$, with $S_U=\widetilde{\mathcal N}_k$, such that
\[
\Pr(\mathcal N_k\ne S_V)\le\Delta_{q,m,r}.
\]
Consequently,
\[
d_{\mathrm{TV}}\!\left(
\mathcal L(\widetilde{\mathcal N}_k,\mathcal N_k),
\mathcal L(S_U,S_V)\right)\le\Delta_{q,m,r},
\]
\[
\abs{\E\frac{\abs{\mathcal N_k\cap\widetilde{\mathcal N}_k}}{k}-\frac{k}{q}}
\le\Delta_{q,m,r},
\]
and, for $k=1$,
\[
\abs{\Pr(\mathcal N_1=\widetilde{\mathcal N}_1)-\frac1q}
\le\Delta_{q,m,r}.
\]
For fixed $q$ and $k$, $m=o(d)$ implies $\Delta_{q,m,d-m}\to0$.
\end{theorem}

\begin{proof}
We decompose $D_j=U_j+V_j$ as in the proof of
\cref{thm:knn-gaussian-limit}. Let $S_U$ and $S_V$ be the indices of the $k$
smallest coordinates of $U$ and $V$. They are independent uniform
$k$-subsets. Set
\[
W_U=\max_jU_j-\min_jU_j,
\qquad G_{V,k}=V_{(k+1)}-V_{(k)}.
\]
If $G_{V,k}>W_U$, then even after adding $U$, every index in $S_V$ remains
ahead of every index outside $S_V$, so $\mathcal N_k=S_V$.

For $i\ne j$,
\[
V_i-V_j=(Z_i-Z_j)^\top(Z_i+Z_j-2Z_0)
\overset{d}=\sqrt{12}\sum_{\ell=1}^rg_\ell h_\ell,
\]
where $g_\ell$ and $h_\ell$ are independent standard normal variables. The
characteristic function of this difference is $(1+12t^2)^{-r/2}$. Fourier
inversion of this function therefore gives
\[
\sup_x f_{V_i-V_j}(x)
\le\frac1{2\sqrt{12\pi}}
\frac{\Gamma((r-1)/2)}{\Gamma(r/2)}.
\]
This density bound implies that
\[
\Pr(\abs{V_i-V_j}\le t)\le c_rt,
\qquad
\Pr(G_{V,k}\le t)\le\binom q2c_rt.
\]
Since $U$ and $V$ are independent, conditioning on $W_U$ yields
\[
\Pr(G_{V,k}\le W_U)\le\binom q2c_r\E W_U.
\]
Marginally, $U_j\sim2\chi_m^2$. Laurent and Massart
\cite{laurent2000adaptive} proved the concentration inequality used here; a
union bound and tail integration give
\[
\E W_U\le12\sqrt{m\log(2q)}+16\log(2q).
\]
These bounds prove the coupling. Because two independent uniform $k$-subsets
have expected normalized overlap $k/q$, the remaining statements follow.
\end{proof}

\section{Classical Deterministic Ordinal-Capacity Bounds}
\label{sec:ordinal}

While our probabilistic theorems in the main article concern random maps and
Gaussian clouds, a classical deterministic question asks whether any
embedding can realize an arbitrary ordering in low dimension. For monotone
Euclidean embeddings, Bilu and Linial used sign-pattern methods to show that
almost every finite metric requires linear dimension
\cite{bilu2005monotone}. We record the corresponding number of full distance
orders and a universal dimension bound for realizing them.

Let $\mathcal O_{n,m}$ be the set of strict total orders of the
$N=\binom n2$ pairwise distances induced by $n$ labeled points in $\R^m$.
Each comparison $\norm{x_i-x_j}^2>\norm{x_k-x_\ell}^2$ is a degree-two
polynomial inequality in $mn$ variables, and there are
$\binom N2=O(n^4)$ such polynomials.

\begin{theorem}[Ordinal capacity]
\label{thm:ordinal}
\begin{romannum}
\item The sign-pattern bounds of Warren \cite{warren1968lower} and Milnor
\cite{milnor1964betti} give
\[
\abs{\mathcal O_{n,m}}\le\exp(O(mn\log n)).
\]
\item Almendra-Hern{\'a}ndez and Mart{\'\i}nez-Sandoval
\cite{almendra2022prescribing} proved that every strict total ordering of the
$N$ pairwise distances is realizable in $\R^{n-2}$. Their February 2026
correction, prompted by a construction of Maldonado, Raggi P\'erez, and
Rold\'an-Pensado \cite{maldonado2026total}, leaves open whether every ordering
is realizable in $\R^{n-3}$.
We call the minimum dimension realizing a specified $n$-point ordering its
\emph{Almendra-Hern{\'a}ndez--Mart{\'\i}nez-Sandoval (AHMS) dimension}.
Thus every such ordering has AHMS dimension at most $n-2$.
\item Consequently, a uniformly random abstract ordering satisfies
\[
\Pr(\text{realizable in }\R^m)
\le\exp[-\Theta(n^2\log n)+O(mn\log n)],
\]
which is superexponentially small when $m=o(n)$.
\end{romannum}
\end{theorem}

The maximal AHMS dimension therefore has order $\Theta(n)$, although its exact
value remains open.

\subsection{Proof of \texorpdfstring{\cref{thm:ordinal}}{the ordinal-capacity compilation}}

For part 1, the $O(n^4)$ comparison polynomials have degree two in $mn$
variables. Warren's sign-pattern bound \cite{warren1968lower} gives
$(Cs\cdot2/(mn))^{mn}=\exp(O(mn\log n))$ realizable patterns. For part 2,
Almendra-Hern{\'a}ndez and Mart{\'\i}nez-Sandoval
\cite{almendra2022prescribing} established the universal $n-2$ realization
bound. For part 3,
$N!=\exp(\Theta(n^2\log n))$ abstract orderings are possible; dividing the
sign-pattern count by $N!$ proves the probability bound. \qed

\begin{conjecture}[Typical AHMS dimension]
\label[conjecture]{conj:ordinal-gaussian}
For Gaussian point clouds whose ambient dimension is at least proportional to
$n$, the AHMS dimension of the induced strict distance ordering is
$\Omega(n)$ with high probability. Thus Gaussian orderings would have the same
linear-dimensional order as the worst case. More quantitatively, the minimum
Kendall distance to $\mathcal O_{n,m}$ should interpolate between the classical counting
obstruction and the pairwise $\sqrt{m/d}$ law. The counting ratio cannot
prove this distribution-specific statement because Gaussian distance
orderings are not uniform.
\end{conjecture}

\section{Reproducibility and Artifact Generation}
\label{sec:supp-reproducibility}

We host the manuscript's executable companion at
\begin{center}
\url{https://github.com/piyush314/random-projection-geometry}.
\end{center}
This companion contains a reusable Python package, a theorem-level
verification harness, paper-specific reproduction programs, experiment
contracts, tutorials, notebooks, and a documentation site. To ensure exact
reproducibility, our manuscript repository includes the companion as a
version-pinned Git submodule under \path{companion/}. The recorded submodule
commit identifies the exact executable version used for the reported results.

\subsection{Environment and theorem checks}

To initialize the pinned companion and install its development dependencies,
run these commands from the manuscript root:
\begin{verbatim}
git submodule update --init --recursive
python -m pip install -e "companion[dev]"
\end{verbatim}
To run the theorem-level harness, execute
\begin{verbatim}
python companion/verification/run_all.py
\end{verbatim}
The harness evaluates each principal theorem independently and reports a pass/fail summary. Its checks combine symbolic or exact-moment identities, deterministic numerical integration, and seeded Monte Carlo experiments, according to the type of claim.

\subsection{Manuscript tables and figures}

Regenerate the numerical rows used in \cref{tab:verify,tab:new-verify} from
the manuscript root by running
\begin{verbatim}
python companion/reproduction/paper/verify_claims.py \
  --seed 42 --full --output-dir data
python companion/reproduction/paper/verify_new_results.py \
  --seed 20260812 --full --output-dir data
\end{verbatim}
The first program checks the scalar channel, anisotropy, and Gaussian maps. The
second checks balanced spectral blocks, the Gaussian integral for a common
minimum, nearest-neighbor overlap, and JL--Kendall coexistence. The
\texttt{--full} option selects the sample sizes reported in the paper; omitting
it gives a faster smoke test. Together, these programs write the
machine-readable results. The first also generates the \LaTeX{} rows used in
\cref{tab:verify}; the compact summary in \cref{tab:new-verify} is maintained in
the manuscript source.

Run the independent Gaussian replacement experiment, including empirical
calibration and the exact disjoint-pair $F_{m,d}$ comparison used in
\cref{prop:zero-info-sharpness}, with
\begin{verbatim}
python companion/experiments/zero_information_jl/run.py \
  --profile full --output companion/artifacts/zero_information_jl
\end{verbatim}
This script writes the paper figure, the complete six-panel diagnostic, CSV
files for individual trials, the comparison of $F$ ratios, and machine-readable
metadata. The commands that build the archival paper bundle and publication
figures invoke the same experiment automatically.

To regenerate the five publication plots directly in the manuscript tree, run
\begin{verbatim}
python companion/scripts/generate_paper_assets.py \
  --output figures/plots
\end{verbatim}
The companion repository also provides reproduction notebooks that expose the same calculations interactively. Its documentation records the expected outputs, experiment contracts, dependency versions, and continuous-integration checks.

\printbibliography[title={References}]

\end{document}